\documentclass[journal]{IEEEtran}
\usepackage{siunitx}
\DeclareSIUnit\mt{\milli\tesla} 
\usepackage[utf8]{inputenc}
\usepackage[english]{babel}
\usepackage{xurl} 
\usepackage{hyperref}
\usepackage{amsmath,amsfonts}
\usepackage[linesnumbered,ruled,vlined]{algorithm2e}
\usepackage{array}
\usepackage{array, makecell} %
\usepackage[noadjust]{cite}

\usepackage{textcomp}
\usepackage{url}
\usepackage{verbatim}
\usepackage{graphicx}
\def\BibTeX{{\rm B\kern-.05em{\sc i\kern-.025em b}\kern-.08em
    T\kern-.1667em\lower.7ex\hbox{E}\kern-.125emX}}
\usepackage{balance}
\usepackage{xcolor}
\usepackage{multirow} 

\usepackage{setspace}
\usepackage{graphicx}
\usepackage{url}
\usepackage{subcaption}
\usepackage{caption}
\usepackage{amsmath}
\usepackage{algcompatible}
\usepackage{array}
\usepackage{amsthm}
\usepackage{adjustbox}
\usepackage{txfonts}
\usepackage{stfloats}

\newtheorem{lemma}{\textbf{Lemma}}

\usepackage{lipsum}  
\usepackage{color}

\usepackage{pgfplots}
\usepackage{tikz}
\usetikzlibrary{fadings}
\usetikzlibrary{patterns}
\usetikzlibrary{shadows.blur}
\usetikzlibrary{shapes} 
\usepackage[linesnumbered]{algorithm2e}%
\let\oldnl\nl
\newcommand{\nonl}{\renewcommand{\nl}{\let\nl\oldnl}}
\pgfplotsset{compat=1.18} 
\begin{document}

\title{ Event-Triggered Reinforcement Learning Based Joint Resource Allocation  for Ultra-Reliable  Low-Latency V2X Communications}

\author{Nasir~Khan,~\IEEEmembership{Member,~IEEE} and~Sinem Coleri,~\IEEEmembership{Fellow,~IEEE}
	\thanks{Nasir~Khan and Sinem Coleri are with the department of Electrical and Electronics Engineering, Koc University, Istanbul, Turkey, email: $\lbrace$nkhan20, scoleri$\rbrace$@ku.edu.tr. This work is supported by Scientific and Technological Research Council of Turkey Grant $\#$119C058 and Ford Otosan.}}
\maketitle
\begin{abstract}
Future  6G-enabled vehicular networks face the challenge of ensuring ultra-reliable low-latency communication (URLLC) for delivering safety-critical information in a timely manner. Existing resource allocation schemes for vehicle-to-everything (V2X) communication systems primarily rely on traditional optimization-based algorithms. However, these methods often fail to guarantee the strict reliability and latency requirements of URLLC applications in dynamic vehicular environments due to the high complexity and communication overhead of the solution methodologies. This paper proposes a novel deep reinforcement learning (DRL) based framework for the joint power and block length allocation to minimize the worst-case decoding-error probability in the finite block length (FBL) regime for a URLLC-based downlink V2X communication system. The problem is formulated as a non-convex mixed-integer nonlinear programming problem (MINLP). Initially, an algorithm grounded in optimization theory is developed based on deriving the joint convexity of the decoding error probability in the block length and transmit power variables within the region of interest. Subsequently, an efficient event-triggered DRL based algorithm is proposed to solve the joint optimization problem. Incorporating event-triggered learning into the DRL framework enables assessing whether to initiate the DRL process, thereby reducing the number of DRL process executions while maintaining reasonable reliability performance. The DRL framework consists of a two-layered structure. In the first layer,  multiple deep Q-networks (DQNs) are established at the central trainer for block length optimization. The second layer involves an actor-critic network and utilizes the deep deterministic policy-gradient (DDPG)-based algorithm to optimize the power allocation. Simulation results demonstrate that the proposed event-triggered DRL scheme can achieve 95$\%$ of the performance of the joint optimization scheme while reducing the DRL executions by up to 24$\%$ for different network settings.
\end{abstract}
\IEEEpeerreviewmaketitle

\section{Introduction} \label{sec:intro}
\noindent \IEEEPARstart{V}{ehicle}-To-Everything (V2X) networks have recently attracted a lot of attention as a crucial enabler for intelligent transportation systems (ITS) by enhancing traffic safety, improving
vehicle coordination, and enabling novel transportation
features such as platooning, advanced real-time navigation,
and lane change alerts. Compared to  dedicated short-range
communications (DSRC) and the long-term evolution (LTE) technology, the 5G New Radio (5G-NR) V2X  offers several advantages, including controllable latency, larger coverage area, robust scalability, and high data rates, even in high-mobility scenarios \cite{5G-V2X}.  While existing and predecessor technologies for V2X applications are mainly designed to cater to data-intensive applications for heterogeneous traffic and infotainment services (video streaming, location tracking, map updates),  the ultra-reliable and low-latency communications (URLLC) aspect is not adequately studied to support safety-critical services such as cooperative adaptive cruise control and intersection collision warnings. Advanced  V2X  use cases, e.g., vehicle platoon, advanced driving,
extended sensors, real-time navigation, and remote driving often demand a transmission latency within a few milliseconds and reliability on the order of 99.999$\%$ \cite{r0}, \cite{r2}. Although new statistical methods for wireless channel modeling  \cite{nilo_new}, and novel techniques for ultra-reliable communications \cite{nilofer1} have been proposed for URLLC system design, addressing the resource allocation in highly dynamic vehicular networks in combination with the strict URLLC  requirements is still a key problem that needs to be addressed.

To meet the stringent quality-of-service (QoS) requirements in the URLLC network design, the control messages are required to be transmitted with short packet size (e.g., 50 $\sim$ 300 bytes) \cite{r9}. Most existing works
on V2X communications are based on Shannon’s capacity
theorem  \cite{2timescale, VURLLC2, VURLLC4} implicitly assuming infinite block length availability. However, 
for the practical realization of URLLC-based V2X applications, the information size is limited, which necessitates incorporating the finite block length (FBL) information theory to reduce the transmission delay. Additionally, for small packet sizes, the decoding error probability is no longer negligible, and the trade-off between the transmission rate, decoding error probability, and transmission latency (in terms of block length/channel uses) cannot be captured by Shannon’s capacity formula. As an important enabler of URLLC,  recent works have adopted the FBL structure for resource allocation in V2X networks with the objective of network-wide power minimization \cite{URLLC4}, the transmission information maximization \cite{r8}, and the worst-case transmission latency minimization \cite{r9}, \cite{worstlatency}. However, these works mainly focus on achieving ultra-low latency, whereas the reliability aspect in terms of minimizing the decoding-error probability is not considered. Additionally, these works  \cite{2timescale, VURLLC2, VURLLC4, r8, r9, URLLC4, worstlatency} overlook optimizing the block length, which significantly impacts the delay performance of latency-sensitive communications.  

The minimization of decoding-error probability has been investigated previously in wireless resource allocation schemes to characterize the reliability performance in the design of energy harvesting networks \cite{r12}, unmanned aerial vehicle (UAV) assisted communication networks \cite{r13}, and industrial automation systems \cite{r14}, \cite{rr15}. In \cite{r12}, the simultaneous wireless information and power transfer (SWIPT) system parameters and block length are concurrently optimized to increase the system reliability. The location and block length of a UAV are jointly optimized while guaranteeing bounds on the target decoding-error probability in \cite{r13}. However, these works do not consider optimizing the power allocation, which has a significant impact on the reliability of the system. In \cite{r14}, the problem of improving the target device reliability while dealing with the transmit power and block length limitations is addressed in an industrial automation scenario. It is demonstrated that the decoding-error probability set for one device determines the achievable error probability performance for the other device, leading to unfair reliability across various devices. \cite{rr15} focuses on achieving fair reliability among the devices by minimizing the maximum decoding-error probability for power allocation and transmission block length optimization in an industrial automation scenario. The aforementioned works \cite{r12,r13, r14, rr15} do not investigate the joint convexity of decoding-error probability with respect to the block length and transmit power, which can further improve the performance of systems requiring bounded latency and extreme reliability \cite{ J_Convexity}. Additionally,  these works adopt iterative optimization theory-based algorithms inherently experiencing high runtime complexity, making them impractical for mission-critical V2X applications.

Considering that the time frame-based decision-making can be regarded as a Markov decision process (MDP), many studies have applied reinforcement learning methods for resource allocation in V2X networks.  Reinforcement learning (RL) effectively addresses decision-making under uncertainty and solves hard-to-optimize objective functions efficiently \cite{DRLforwireless}. Deep RL (DRL) combines RL with deep neural networks (DNN) to deal with the large state-action space, making it particularly appealing for real-time resource management. DRL  has been widely adopted to address resource allocation problems in  URLLC-aware V2X networks scenarios \cite{DRL23_3, DRL23_5, DRL23_6, DDPG0, DRL23_1}. \cite{DRL23_3} proposes a federated multi-agent DRL scheme for channel selection and power allocation in a decentralized fashion by incorporating the reliability constraint in terms of the signal-to-interference-plus-noise ratio (SINR) outage probability on the vehicle-to-vehicle (V2V) communication links.
 \cite{DRL23_5} considers a two-timescale federated DRL algorithm combined with spectral clustering for transmission mode selection in a V2X communication scenario, where the reliability and latency constraints in terms of tolerable
outage probability are incorporated into the constrained optimization problem. A multi-agent DRL scheme to allocate resources for both unicast and broadcast V2X applications in the absence of the global channel state information (CSI) is addressed in \cite{DRL23_6}, where the URLLC constraints are incorporated into the reward function directly. \cite{DDPG0}  proposes a hybrid centralized-distributed DRL-based resource allocation scheme by considering both the total system capacity and reliability of the vehicular links in the design objective. \cite{DRL23_1} formulates a multi-objective resource allocation problem to maximize the transmission success ratio and the mean opinion score of vehicular links for intra-platoon communications using a novel DRL-based framework, where the latency requirement is incorporated into the observation space of the DRL agent as the percentage of the remaining time
over the transmission latency. 
For the mentioned literature on URLLC-enabled V2X systems \cite{DRL23_3, DRL23_5, DRL23_6, DDPG0, DRL23_1}, only V2V communication is utilized for disseminating crucial safety messages among vehicles, whereas the vehicle-to-infrastructure (V2I) links are dedicated to data-intensive services and non-safety infotainment applications using
infinite block length transmissions. Nevertheless, the reliability of the V2V links diminishes when considering blockage effects, which hampers the performance of V2V communications. Additionally, advanced URLLC-based V2X use cases (e.g.,  teleoperated driving and vehicle platooning) often rely on V2I links for remote control \cite{5G-NR}, \cite{inband}.  Moreover, the resource allocation strategies \cite{DDPG0, DRL23_1,  DRL23_3, DRL23_5, DRL23_6}  may not be suitable for realistic URLLC-6G resource management problems in actual wireless systems because the learning process is conducted  in a time-triggered mechanism, i.e., the resource allocation decisions are updated periodically in each time frame, requiring computationally intensive training procedures.

 Against the time-triggered mechanism, event-triggered learning has been proposed with the goal of activating computations or communications only when the system state deviates from the expected accuracy or when the triggering condition is violated \cite{ET}. Recently, event-triggered learning and  DRL have been combined for path planning in an autonomous driving scenario \cite{ETO}. The proposed methodology involves incorporating the triggering condition as communication loss into the reward function, learns the driving control policy for autonomous path planning, and is able to improve the system performance compared to a standard DRL-based learning approach. None of these studies however incorporate event-triggered learning into a DRL approach for minimizing the computational expenses and optimizing the continuous power and integer block length allocation in a URLLC-aware V2X network. 
 
\subsection{Contributions and organization}
This paper proposes novel adaptive solution strategies to jointly optimize the transmit power and block length allocation with the objective of minimizing the worst-case decoding-error probability among all the vehicles subject to the stringent reliability and latency requirement for a time-critical V2X URLLC scenario. A preliminary version of this work proposing a sub-optimal algorithm based on the block coordinate descent (BCD) method for the block length and transmit power optimization in a vehicular communication scenario is described in \cite{myconf}.  Different from our previous work \cite{myconf}, we first analyze the joint-convexity of the optimization problem with respect to the block length and transmit power variables, derive the Karush–Kuhn–Tucker (KKT) optimality conditions to solve the joint optimization problem, and analyze its applicability in a vehicular network. Additionally, we propose a  DRL-based robust algorithm for the block length assignment and power allocation problem, which can effectively handle the excessive computational complexity associated with the joint optimization scheme. 
Further, we incorporate event-triggered learning into the proposed DRL-based framework enabling the assessment of whether to initiate the DRL process. This, in turn, leads to a decrease in the frequency of DRL process executions, all the while upholding acceptable levels of reliability performance.

The original contributions of this paper are listed as follows:

\begin{itemize}
\item We propose a novel optimization theory based solution strategy for the joint optimization of code block length and transmit power with the goal of minimizing the maximum decoding error probability of all vehicles while considering the transmit power, reliability, and latency constraints in a URLLC-enabled vehicular network, for the first time in the literature. We show that the decoding error probability function preserves joint convexity in block length and transmit power within reasonable limits on the block length. We then relax the integer block length constraint, derive optimality conditions by using the KKT analysis, solve the problem by using Lagrangian dual decomposition method and utilize a low-complexity greedy search methodology to convert the positive continuous block length values to an integer solution. 

\item We propose a centralized event-triggered DRL-based scheme to determine the optimal block length and transmit power allocation policies to achieve fair reliability while satisfying the URLLC constraints, for the first time in the literature. For the DRL scheme, a two-layered training framework is designed. In the first layer, a multi-DQN structure is used for code block length assignment separately. The second layer involves an actor-critic network and utilizes the deep deterministic policy-gradient (DDPG)-based algorithm to optimize the transmit power allocation. 
To accelerate learning efficiency and relieve the trainer of computation stress, we leverage event-triggered learning by introducing a trigger block based on the input state similarity.

\item Via extensive simulations, we demonstrate that our proposed event-triggered DRL-based approach outperforms the proposed joint optimization method and the recent state-of-the-art DRL-based approach utilizing a single-DQN to optimize the joint resource allocation 
in terms of computational complexity and convergence time for different network configurations.
\end{itemize}

The rest of the paper is organized as follows. Section \ref{sec:system} describes the vehicular URLLC system model and assumptions used in the paper. Section \ref{sec:opt_prob} presents the mathematical formulation of the worst-case decoding-error probability minimization problem for the vehicular URLLC system. Section \ref{sec:joint} describes the joint-convexity analysis and the optimization theory based solution methodology for the transmit power and block length optimization problem. Section \ref{reinforce} provides a brief background on the design of the DRL algorithms and describes the proposed event-triggered DRL-based algorithm for the transmit power and block length allocation problem. Section \ref{sec:simulation} evaluates the performance of the proposed solution strategy. Finally, conclusions and future research directions are presented in Section \ref{sec:conclusion}.

\section{System Model and Assumptions} \label{sec:system}
We consider a single-cell  V2I   URLLC system,  where vehicles move on a multi-lane highway with the RSU at the center and a fixed distance from the highway, as shown in Fig. \ref{fig:V2V}. 
The roadside unit (RSU) plays the role of the 5G base station (gNB) for centralized radio resource management \cite{5G-NR} and sends unicast safety-critical information to the vehicles for coordination and safety alerts using short packets. 
We consider a practical scenario where this RSU is deployed in a rural area and deprived of a permanent grid source. We assume the RSU is equipped with large batteries that are periodically recharged using energy harvesting techniques (wind/solar) \cite{rural_battery}.
We assume that the RSU and the vehicle user equipment (VUEs) are equipped with a single transmit-receive antenna. The set of VUEs under the coverage of the RSU is denoted by   $\mathcal{K}=\{1, \ldots, K\}$. This system model represents different  URLLC V2X scenarios, such as time-sensitive High-Definition (HD) map data dissemination for vehicle coordination and localization in autonomous driving, coordinating the trajectories of vehicles to make better maneuvers in remote/teleoperated driving, and cooperative collision avoidance for vehicle safety in practical traffic management scenarios \cite{SC1}.

In NR-V2X mode-1 operation, the RSU centrally allocates the resources for V2X communications. This mode requires the vehicles to report their CSI to the gNB when initiating the V2X communications.  The RSU transmits the CSI reference signal (CSI-RS) through the Physical downlink Control Channel (PDCCH)  for downlink CSI acquisition. Then,  the VUE  decodes the received CSI-RS to acquire channel measurements. Subsequently,  the VUE reports the CSI  to the RSU using the Physical Uplink Control Channel (PUCCH).  The  CSI transmission introduces additional overhead, which can be mitigated by incorporating sensing-assisted communications \cite{ISAC}.  In this approach, the RSU can directly acquire the downlink CSI  based on the vehicles' echo signals, capturing the vehicles' motion parameters from the echo and analyzing the channel quality based on the power of the echo signal.

\begin{figure}[t]
 \centering
\includegraphics[width= 0.9\linewidth]{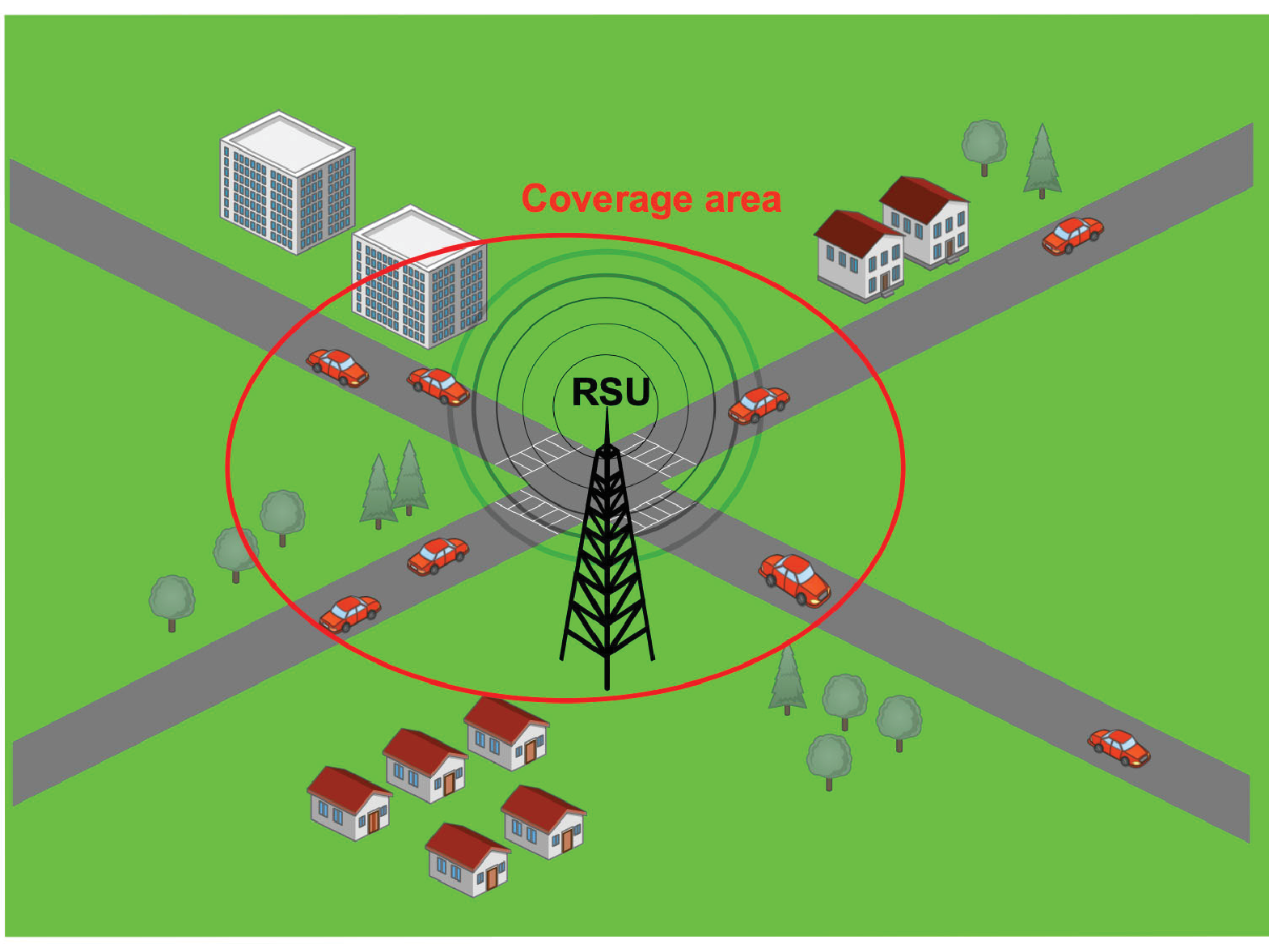 }
\caption{System model for the V2X communication network.} \label{fig:V2V}
\end{figure}
 We consider a slotted communication system with the time period for information transmission discretized into a series of time frames $t \in \{1,2, \cdots, \mathcal{T}\}$.
 The  URLLC-aware resource allocation algorithm sends $K$ short data packets to the $K$  VUEs in each time frame,  and the transmission of these packets is handled in a time-division-multiple-access (TDMA) fashion such that each VUE is allocated with unique block lengths.
To fulfill the requirements of latency-sensitive communications, scheduling is executed at the beginning of each time frame. The entire downlink information transmission is subject to strict restrictions on the latency, requiring completion within a single time frame spanning $M_{D}$ symbols such that  $\sum_{k \in \mathcal{K}} m_{k}^{(t)} \leq M_{D}$, where $ m_{k}^{(t)}$ is the channel block length allocated to the $k^{th}$ VUE. Here, $M_D$ represents the maximum block length, indicating the maximum number of symbols transmitted within a time frame, and ensures timely delivery of all short packets to vehicles within the maximum transmission delay  $T_{max}=\frac{M_D}{B}$, where $B$ is the system bandwidth. This aggregate delay assumption ensures that the packets generated for all the VUEs during a time frame are scheduled in the next frame, and this assumption is considered in several other works \cite{r13, r14, rr15}.  Moreover, the continuous transmit power allocated to the $k^{th}$ VUE  in time frame $t$, denoted by $P_{k}^{(t)} \geq 0$, is subject to the constraint $ \left(\sum_{k=1}^{K}\frac{ P_{k}^{(t)} m_{k}^{(t)}}{M_{D}}\right) \leq P_{\max }$, where $\frac{m_{k}^{(t)}}{M_{D}}$ is the fractional block length allocated to the $k^{th}$ VUE, and $P_{\max }$ is the RSU  power budget. 

The downlink channel gain from the RSU to the $k^{th}$ VUE in time frame $t$ is denoted by   
$h_{k}^{(t)}=\alpha_{k}^{(t)}\left|g_{k }^{(t)}\right|^2$, where $\alpha_{k}^{(t)}$ denotes the large-scale fading (path loss and log-normal shadowing) and $g_{k}^{(t)}$  denotes the complex Gaussian random variable with Rayleigh distributed envelope.  We utilize a quasi-static flat fading channel model where the channel remains constant in the coherent time
$\mathcal{T}_c$  but might change with correlation patterns across consecutive time frames.  The  channel coherence time is given by $\mathcal{T}_c=\sqrt{\frac{9}{16 \pi f_D^2}},$ where $f_D = \frac{f_c \cdot v_s}{c}$ is  the Doppler frequency, $f_c$ is the carrier frequency, $v_s$ is the vehicle velocity, and $c$ is the speed of light \cite{channel_coherence}. For moderate vehicular speeds, the transmission time interval, considered equal to the duration of one frame, is far less than the channel coherence time, i.e., $\mathcal{T}_c > t$ holds. Thus,  the CSI can be regarded as constant throughout the information transmission period, i.e., the channel remains static over $M_D$ symbols.  We use a first-order Gauss-Markov process to describe the relationship between the small-scale Rayleigh fading in two consecutive time frames according to Jake's statistical model \cite{rr21}:
\begin{equation}\label{jake}
g_{k}^{(t)}=\rho\hspace{0.03cm} g_{k}^{(t-1)}+w_{k}^{(t)},
\end{equation}
where $\rho\,(0 \leq \rho \leq 1) $ is the correlation coefficient between two consecutive small-scale fading realizations, $w_k{^{(t)}}\sim\mathcal{CN} \left(0,1-\rho^{2}\right)$ is the channel discrepancy term, $g_{k}^{(0)}\sim\mathcal{CN} (0,1)$ and $\mathcal{CN }$ represents a circularly symmetric complex normal  Gaussian random. The correlation coefficient is given by $\rho=J_{0}\left(2 \pi f_{D} T\right)$, where $J_{0}(.)$ is the zeroth-order Bessel function,   and $T$ is the time interval over
which the correlated channel variation occurs.  

In the context of finite block length regime,  the achievable information rate for the $k^{th}$ VUE, which can be decoded with decoding error probability no greater than $\varepsilon_{k}^{(t)}$  is given by the normal approximation \cite{r3} 
 \begin{equation}\label{eq:rate}
     R_{k}^{(t)}\approx \left(\log _{2}\left(1+\gamma_{k}^{(t)}\right)-\sqrt{\frac{V_{k}^{(t)}}{m_{k}^{(t)}}} \frac{f_{Q}^{-1} \left(\varepsilon_{k}^{(t)}\right)}{\ln 2}\right),
 \end{equation}
where  $V_{k}^{(t)}\triangleq1-(1+\gamma_{k}^{(t)})^{-2}$ is the channel dispersion, $\varepsilon_{k}^{(t)}\in(0,1)$ is the  decoding error probability, $\gamma_{k}^{(t)}= \left(\frac{P_{k}^{(t)} h_{k}^{(t)}}{\sigma^{2}}\right)$ is the signal-to-noise-ratio (SNR) for the  $k^{th}$  VUE in time frame $t$, $\sigma^{2}$ is the noise power and  $f_{Q}^{-1}(\cdot)$ is the inverse of the function $f_{Q}(x) = \frac{1}{\sqrt{2\pi}} \int_x^{\infty} e^{-\frac{t^2}{2}} \, dt$. Eq. (\ref{eq:rate}) is a non-asymptotic approximation for the achievable rate for finite block length codes. Practical coding schemes (e.g., extended BCH, LDPC, Polar codes)  closely approach the normal approximation while maintaining a relatively constant gap in the block error rate to the finite block length fundamental limit \cite {short_LDPC}. Therefore, we consider Eq. (\ref{eq:rate}) a tight and accurate approximation for the maximal coding rate under finite block length transmissions.

We consider a fixed payload size of $L$ bits per vehicle. For a given block length $m_k^{(t)}$, the  coding rate for the $k^{th}$ VUE in time frame $t$ becomes $R_{k}^{(t)}=\frac{L}{m_{k}^{(t)}}$ [bits/seconds/Hz]. By rearranging Eq. (\ref{eq:rate}), the decoding error probability at the $k^{th}$ VUE can be expressed as
\begin{equation}\label{eq:error}
\begin{aligned}
\varepsilon_{k}^{(t)}\Big(\gamma_{k}^{(t)}, m_{k}^{(t)}, L\Big) \approx f_{Q}\left(  \frac{\sqrt{m_{k}^{(t)}} \ln 2}{\sqrt{V_{k}^{(t)}}}\left(C_k^{(t)}-\frac{L}{m_{k}^{(t)}}\right) \right),
\end{aligned}
\end{equation}
where $C_k^{(t)}=\log _{2}\left(1+\gamma_{k}^{(t)}\right)$ is the Shannon rate. For  typical URLLC-enabled  V2I
communications for traffic efficiency/safety, the codewords are required
to be short \cite{channel_coherence}, which can easily satisfy the URLLC requirement subject to the availability of sufficient SNR. 
When the instantaneous CSI of the downlink channel is poor, i.e., $\gamma_{k}\leq \gamma_{th}=0$ dB, the reliability performance cannot be guaranteed even with maximally available block length. In such a scenario, we assume that the RSU does not allocate resources and leaves it to the next time frame.

\section{Worst-case Decoding Error Probability Minimization Problem} \label{sec:opt_prob}
In this section, we present the mathematical formulation of the \emph{Min-Max} decoding-error probability problem denoted by $\cal{MMDEP}$. In this section, we omit the index  $(t)$ for
notational simplicity.

The joint optimization of the transmit power and block length allocation to minimize the worst-case decoding error probability under stringent latency and reliability constraints is formulated as follows:

$\cal{MMDEP}$:
\begin{subequations} \label{opt_problem} 
\begin{align}
& &  \min_{ \boldsymbol{P}{},\hspace*{0.02cm}\boldsymbol{m}{} }& \quad \max _{k \in {K}}\hspace*{0.2cm} \varepsilon_{k}
\label{obj}\\
& \textit{subject to}
&& \sum_{k=1}^{K} \frac{m_kP_{k}}{M_{D}}\leq P_{max}   \label{pmax} \\
&&&  \sum_{k=1}^{K}m_{k}\leq  M_{D}, \hspace*{0.1cm}  \label{latency}\\
&&& \varepsilon_{k}\leq \varepsilon_{max, k}, \hspace*{0.1cm}  k \in \{1,...,K\} \label{reliability} \\
& \textit{variables}
& & P_{k} \geq 0,\hspace*{0.1cm} m_{k} \in \mathbb{Z}^{+}.
\label{pcp_vars}
\end{align}
\end{subequations}
where  $\boldsymbol{P}{}$ and  $\boldsymbol{m}$ are the transmit power and block length allocation vectors of the network, i.e., $\boldsymbol{P}{}=\left[P_1{}, \ldots, P_K\right]$ and $\boldsymbol{m}=\left[m_1{}, \ldots, m_K\right]$, and $ m_{k} \in \mathbb{Z}^{+}$ indicates the set of all positive integers. The variables of the optimization problem are $ {P_{k}} $, the transmit power allocated to the $k^{th}$ VUE; ${m_{k}}$, the block length in units of symbols or channel uses allocated to the $k^{th}$ VUE, \hspace*{0.05cm}$ \forall k \in \left\{ 1, \cdots, K\right\}$.

The objective of the optimization problem is to minimize the worst-case decoding error probability for the vehicular network as given by  (\ref{obj}).  (\ref{pmax})  represents the total power budget constraint, ensuring that the sum of the product of fraction of the time allocated to each VUE and its corresponding power allocation is within the total power budget of the RSU. (\ref{latency}) gives the latency constraint: The total block length assigned to all the VUEs should be less than or equal to the maximum number of symbols allowed to conduct the information transmission.  To ensure the quality of transmission, (\ref{reliability}) represents the upper bound of tolerable packet error probability for each VUE, with $\varepsilon_{max, k}$ denoting the corresponding quality of service requirement for the $k^{th}$ vehicle. Finally, (\ref{pcp_vars}) represents a non-negative transmit power constraint and the positive integer constraint for the block length assigned to VUE.

The optimization problem under consideration is a non-convex mixed integer nonlinear programming (MINLP) due to the non-convex objective function  (\ref{obj}), non-convex constraint  (\ref{pmax}) and the integer constraint for the block length  (\ref{pcp_vars}). Therefore, directly achieving the global optimal solution is extremely difficult, i.e., finding a global optimum requires algorithms with exponential complexity. Next, we propose an optimization theory based solution and an event-triggered DRL-based framework to solve this problem.
\section{ Joint Optimization of Power and block length} \label{sec:joint}
In this section, we reformulate the non-convex optimization problem $\cal{MMDEP}$ using the variable transformation method and relaxation of the integer block length constraint. The reformulated optimal power and block length problem is denoted by $\cal{OPBP}$. We show that within the region of interest, $\cal{OPBP}$ is jointly convex with respect to block length and transmit power as decision variables. We then derive the optimality conditions using KKT analysis and propose a Joint Optimization Algorithm ($\cal{JOA}$) to solve this problem. The $\cal{JOA}$ efficiently solves $\cal{OPBP}$ using standard convex optimization tools. It further employs a low-complexity suboptimal algorithm to convert the block lengths into integer values.

To facilitate the derivation for solving the $\cal{OPBP}$, we first relax the integer constraint for block length in (\ref{pcp_vars}), i.e., $m_{k}\in \mathbb{R}^{+}$. Second, we exploit the variable transformation method, which is used to illustrate the joint-convexity of $\cal{OPBP}$ in the decision variables. In particular, we define $a_{k}\triangleq \frac{1}{m_{k}}$ and $b_{k}\triangleq {\sqrt{P_{k}}}$. Additionally, to handle the non-differentiable min-max objective function, we apply the epigraph transformation and minimize an upper-bound $\eta$ as an alternate objective for minimizing the maximum decoding-error probability. This transformation facilitates the solution design while keeping the original objective and constraints intact. Then, the $\cal{OPBP}$ can be  mathematically reformulated as follows: 

$\cal{OPBP}$:
\begin{subequations} \label{opt_problem_epigraph} 
\begin{align}
&&&  \min_{ \boldsymbol{a},\boldsymbol{b},\mathbf{\eta} } \quad\eta \label{obj1}\\
& \textit{subject to}
&&  \sum_{k=1}^{K}\frac{ {b_{k}^2} }{a_{k}}\leq P_{max}M_{D}\hspace*{0.1cm} \label{pmax1} \\ 
&&&  \sum_{k=1}^{K} \frac{1}{a_{k}}\leq  M_{D}, \label{latency1}\\
&&& \varepsilon_{k}\left(a_{k},b_{k}\right)\leq \eta, \hspace*{0.1cm}  k \in \{1,...,K\} \label{reliability1} \\
&&& \varepsilon_{k}\left(a_{k},b_{k}\right)\leq \varepsilon_{max, k}, \hspace*{0.1cm}  k \in \{1,...,K\} \label{reliability11} \\
& \textit{variables}
& &  a_{k} \geq 0,\hspace*{0.1cm} b_{k}\geq 0, \hspace*{0.1cm} \eta \in {[0,1]} , \hspace*{0.1cm}  k \in \{1,...,K\}.\label{pcp1_vars}
 \end{align}
\end{subequations}

 where   $\boldsymbol{a}{}=\left[a_1{}, \ldots, a_K\right]$ and $\boldsymbol{b}=\left[b_1{}, \ldots, b_K\right]$, and  $\eta$, $a_{k}$ and $b_{k}$, $\forall k \in \{1, \ldots, K\}$ are the decision variables of $\cal{OPBP}$. To this end, we first provide the following Lemma as a first solution step for solving the $\cal{OPBP}.$

\begin{lemma} \label{lemma_convexity}
Under the assumption that  $\gamma_{k}\left(b_{k}\right)\geq \gamma_{th}=0$ dB and the Shannon rate exceeds the coding rate, i.e., $C_{k}\left(\gamma_{k}\left(b_{k}\right)\right)- a_{k}L >0 $, the  $\cal{OPBP}$ is a convex optimization problem when $m_{k}$ satisfies:
\begin{equation} \label{eq:bl_constraint} 
m_{k} \leq \min \left\{{ 5\ln\left(2\right) \gamma_{k}\left(b_{k}\right) L }\hspace{0.09cm},\hspace{0.09cm} \frac{ 3\ln{\left(2\right)}10\sqrt{10\gamma_{k}\left(b_{k}\right)}L}{4}\right\},
\end{equation}
\end{lemma}

\begin{proof}
Since $f_{Q}(x)$ is a decreasing function of its argument and $\varepsilon_{k} < \varepsilon_{\text{max}, k} < f_{Q}(0) = 0.5$, a necessary condition for constraint (\ref{reliability11}) to hold is $C_{k}\left(\gamma_{k}\left(b_{k}\right)\right) - a_{k}L  > 0, \forall k \in \mathcal{K}$. Further, let us define  
$\beta\left(a_{k},b_{k}\right) \triangleq \left( \frac{ \ln 2}{\sqrt{a_{k}}\sqrt{V_{k}\left(b_{k}\right)}}\left( C_{k}\left(\gamma_{k}\left(b_{k}\right)\right) - a_{k}L \right)\right)$.
Then, $\beta\left(a_{k},b_{k}\right)$ is non-negative and the second-order derivative of  $\varepsilon_{k}\left(a_{k},b_{k}\right) = Q\left(\beta\left(a_{k},b_{k}\right)\right)$ is given by 
\begin{equation}
  \frac{\partial^{2} \varepsilon_{k}\left(a_{k},b_{k}\right)}{\partial \beta^{2}} = \frac{\beta}{\sqrt{2 \pi}} e^{-\frac{\beta^{2}}{2}} > 0.
\end{equation}

Therefore, the decoding error probability decreases monotonically and is a convex function with respect to  $\beta$. Furthermore,  under the condition that $\gamma_{k}\left(b_{k}\right)\geq \gamma_{th}=0$ dB, the first principal minor and the determinant of the Hessian matrix of 
$\varepsilon_{k}\left(\beta\left(a_{k},b_{k}\right)\right)$ are non-negative when (\ref{eq:bl_constraint}) is satisfied . Therefore, $\varepsilon_{k}\left(\beta\left(a_{k},b_{k}\right)\right)$ is jointly convex in the auxiliary variables $a_{k}$ and $b_{k}$ in the considered high reliability scenario. Consequently, the  $\cal{OPBP}$ is a convex optimization problem considering the linearity of the objective (\ref{obj1}) and convexity of the constraints (\ref{pmax1})-(\ref{reliability11}) in variables $a_{k}$, $b_{k}$ and $\eta$. 
\end{proof}

As $\cal{OPBP}$ is a convex optimization problem, it can be solved by analyzing its KKT conditions to reach an optimal solution.
The corresponding Lagrangian function for (\ref{opt_problem_epigraph}) can be written as
\begin{equation}\label{lag_func}
\begin{aligned}
L\left(a_{k},b_{k},\eta,\lambda_{i,k}\right)= 
& \Bigg[\eta +\lambda_{1}\Bigg(\sum_{k=1}^{K}\frac{b_{k}^{2}}{a_{k}}-P_{max}M_{D}\Bigg)\hspace*{0.15cm} + \\
& \lambda_{2} \Big(\sum_{k=1}^{K}\frac{1}{a_{k}} -  M_{D}\Big) + \sum_{k=1}^{K} \lambda_{3,k}\left( \varepsilon_{k}-\eta\right)+ \\
& \sum_{k=1}^{K} \lambda_{4,k}\left( \varepsilon_{k}-\varepsilon_{max, k}\right)\Bigg].
\end{aligned}
\end{equation}

Then, the KKT conditions of $\cal{OPBP}$ are as follows:

\small
\begin{subequations} \label{eqs:gradients} 
\begin{align}
&\Bigg\{\left(\lambda_{3,k}+\lambda_{4,k}\right)\left(\frac{m_{k}^{2}\ln{2} \exp{\frac{-\beta^{2}}{2}}}{\sqrt{2\pi}}\right)\left(\frac{1}{2} m_{k}^{-\frac{1}{2}} V_{k}^{-\frac{1}{2}} C_{k} + \frac{1}{2} m_{k}^{-\frac{3}{2}} V_{k}^{-\frac{1}{2}} L\right) \nonumber \\
&-\lambda_{1}\left(\sum_{k=1}^{K}P_{k}m_{k}^{2}\right)\hspace*{0.15cm}-\lambda_{2}\left(\sum_{k=1}^{K}m_{k}^{2}\right)\Bigg\}=0,\hspace*{0.2cm} k \in \{1,...,K\},\label{gra1} \\
&\Bigg\{(\lambda_{3,k}+\lambda_{4,k})\left(\frac{\sqrt{m_{k}} \sqrt{P_{k}}(H_{k}^{3}P_{k}^{2} +2H_{k}^{2}P_{k}-H_{k}C_{k})+\frac{L\sqrt{P_{k}}H_{k}}{\sqrt{m_{k}}}}{(1+H_{k}P_{k})^3}  \right)\nonumber \\
&\times\left(\frac{-2\ln{2}V_{k}^{-\frac{3}{2}} \exp{\frac{-\beta^{2}}{2}}}{\sqrt{2\pi}}\right)  +\lambda_{1}\left(2\sum_{k=1}^{K}\sqrt{P_{k}}m_{k}\right) \Bigg\}=0,\hspace*{0.2cm} k \in \{1,...,K\}, \label{gra2} \\ 
&\Bigg\{1-\sum_{k=1}^{K}\lambda_{3,k}\Bigg\}=0,\hspace*{0.2cm} k \in \{1,...,K\},\label{gra3}
\end{align}
\end{subequations}
\normalsize
\begin{subequations}
\label{eqs:comp}
\begin{flalign}
& \lambda_{1}\left(\sum_{k=1}^{K} P_{k}m_{k}- P_{max}M_{D}  \right)=0, \hspace*{0.1cm} k \in \{1,...,K\},  \label{comp1} \\
& \lambda_{2} \left(\sum_{k=1}^{K}m_{k} -  M_{D}\right)=0,  \label{comp2} \\
& \lambda_{3,k}\left(\varepsilon_{k}- \eta\right)=0, \hspace*{0.1cm} k \in \{1,...,K\},  \label{comp3} \\
& \lambda_{4,k}\left(\varepsilon_{k}- \varepsilon_{max, k}\right)=0, \hspace*{0.1cm} k \in \{1,...,K\},  \label{comp44}
\end{flalign}
\end{subequations}
where $H_{k}=\frac{h_{k}}{\sigma^{2}}$.  (\ref{gra1})-(\ref{gra3}) and  (\ref{comp1})-(\ref{comp44}) represents the stationarity and complementary slackness conditions, respectively, and $\lambda_{i,k}\geq 0$ is the Lagrange multiplier associated with the $i^{th}$ constraint of the $k^{th}$ VUE, $i \in \{1,2,3,4\}$ corresponding to constraints (\ref{pmax1})-(\ref{reliability11}), respectively.  As the $\cal{OPBP}$ is a convex optimization problem as per Lemma \ref{lemma_convexity}, which satisfies the slater’s constraint qualification condition, the KKT conditions are sufficient conditions for an optimal solution of the $\cal{OPBP}$. Therefore, the optimal solution for block lengths $\boldsymbol{m}=\left\{m_{1}, \cdots, m_{K}\right\}$, and transmit powers $\boldsymbol{P}=\left\{P_{1}, \cdots, P_{K}\right\}$ can be readily obtained by using standard convex solvers. 

Since the block length solution obtained from the relaxed integer problem may violate the original integer constraint (\ref{pcp_vars}), a low-complexity greedy search methodology is adopted to convert the positive continuous block length values to integer solution \cite{r22}. The $\cal{JOA}$ for solving $\cal{MMDEP}$ problem is summarized in Algorithm \ref{convex_algo}. The  $\cal{JOA}$ starts by solving the convex $\cal{OPBP}$ problem (\ref{opt_problem_epigraph}) to generate the non-integer block length vector denoted  by $\boldsymbol{m}^o= \left[ m_{1}^{o}, \cdots, m_{K}^{o}\right]$ (Line $1$). The integer solution for block length vector is then initialized to $\boldsymbol{\tilde{\boldsymbol{m}}_k}=  \left[ \tilde{{m}}_1 , \cdots,  \tilde{{m}}_{K}\right] = \lfloor m_{1}^{o}, \cdots, m_{K}^{o}\rfloor$, where  $\lfloor\cdot\rfloor$ indicates the floor function (Line $2$). Then,  there are  $M_{\text {UA }}=\sum_{k=1}^K m_k^{o}-\sum_{k=1}^K \tilde{{m}}_k$ unassigned block lengths. Since the decoding-error probability is a decreasing function of the block length $m_{k}$, the $\mathcal{JOA}$ allocates one unassigned symbol to the VUE, which results in the largest decrement of the decoding-error probability, i.e., $\text{arg}\min_{k\in K  } \left(\varepsilon_{k}(\tilde{m}_{k}+1)-\varepsilon_{k}(\tilde{m}_{k})\right), \forall k \in \{1,...,K\}$. This process is repeated till all unassigned block lengths are allocated (Lines $4-8$). The $\cal{JOA}$ terminates after updating the final transmit power solution  $\boldsymbol{\tilde{P}}$  based on the integer block length solution $\boldsymbol{\tilde{m}}$ obtained via the greedy search method (Line $9$).

\begin{algorithm}
\caption{$\cal{JOA}$ for Solving $\cal{MMDEP}$ Problem } \label{convex_algo}
\begin{algorithmic}[1]
\STATE solve the convex problem (\ref{opt_problem_epigraph}) to generate \hspace{0.001cm} $\boldsymbol{m}^o= \left[m_{1}^{o}, \cdots, m_{K}^{o}\right]$. 
\STATE obtain initial integer block length solution $\boldsymbol{\tilde{\boldsymbol{m}}_k}= \lfloor m_{1}^{o}, \cdots, m_{K}^{o}\rfloor$.
\STATE determine unassigned block lengths  $M_{\text {UA }}=\sum_{k=1}^K m_k^{o}-\sum_{k=1}^K \tilde{{m}}_k$.
\WHILE {$M_\text{UA }>$\hspace{0.1cm}0}
\STATE $k^{\star} \leftarrow \text{arg}\min_{k\in K  } \left(\varepsilon_{k}(\tilde{m}_{k}+1)-\varepsilon_{k}(\tilde{m}_{k})\right)$,
\STATE $\tilde{m}_{k^{\star}}$ $\leftarrow$ $\tilde{m}_{k^{\star}} + 1$,
\STATE $M_\text{UA }$ $\leftarrow$ $M_\text{UA }-1$,
\ENDWHILE
\STATE obtain $\boldsymbol{\tilde{P}}$ using $\boldsymbol{\tilde{m}}$.
\end{algorithmic}
\end{algorithm} 
 The primary computational complexity of the proposed  $\cal{JOA}$ comes from the concurrent solution of the block length and power allocation vectors with a computation cost of  $\mathcal{O}\left( (K^2+1)\right)$ for $K$ VUEs within the network \cite{convergence}. Additionally, the integer conversion of block length variable via the proposed greedy search method necessitates the iterative allocation of $M_\text{UA}$ unassigned symbols to the $K$ VUE, with a computation cost of $\mathcal{O}\left( M_\text{UA}\right)$. Thus, the overall computational complexity is  $\mathcal{O}\left( (K^2+1) + M_\text{UA}\right)$.

To ensure optimality, it is crucial to determine the transmit power and block length at the beginning of each time frame. However, ensuring timely resource allocation in URLLC V2X environments using the $\cal{JOA}$ presents three primary challenges. First, the time complexity of $\cal{JOA}$ grows quadratically with network size, rendering it impractical for large-scale networks where solutions quickly become outdated due to short coherence time intervals. Second, $\cal{JOA}$ relies on instantaneous global CSI, which may be challenging to acquire in dynamically changing vehicular networks. Third, $\cal{JOA}$ considers a myopic view of the V2X environment, focusing on optimizing individual time slots without considering historical network state information. This approach may lead to sub-optimal solutions and poor performance. 

To address these challenges, an RL-based strategy offers long-term benefits by modeling the optimization problem as a sequential decision-making process. Using an RL-based framework, the network’s historical wireless data can effectively optimize current choices in real-time, a capability absent in conventional optimization-based solutions like the $\cal{JOA}$. Next, we propose a robust event-triggered DRL-based strategy to solve the $\cal{MMDEP}$.

\section{Event-Triggered DRL-Based Power and block length allocation Scheme }\label{reinforce}
 In this section, we propose a DRL-based design to solve the $\cal{MMDEP}$ problem. To build up the foundation for the proposed DRL-based strategy, a  brief overview of the two related DRL methods, Deep Q-network (DQN) and Deep deterministic policy gradient (DDPG), is first provided, followed by the description of the state, action, reward function, for transforming the joint resource allocation problem into an RL framework. An event-triggered mechanism is introduced based on input state similarity to determine whether to trigger the DRL process or not.  Finally, the event-triggered DRL-based algorithm for the optimization of transmit power and block length is proposed.

\subsection{Deep Reinforcement Learning Overview}\label{DRL_overview}
Reinforcement learning is a decision-making process where an agent learns the outcome of its actions over time through repeated trial-and-error interactions with its environment \cite{RL_book}. Let $\mathcal{S}$  denote a set of
possible states and $ \mathcal{A}$ indicate the set of possible actions. In each time frame, the agent interacts with the environment, observes the current state $s^{(t)} \in \mathcal{S}$, takes action $a^{(t)} \in  \mathcal{A}$ according to the current policy $\pi\left(s^{(t)}, a^{(t)}\right)$, where $\pi\left(s^{(t)}, a^{(t)}\right)$: $\mathcal{S} \mapsto \mathcal{A}$ is the probability of taking action $a^{(t)}$ given the current state $s^{(t)}$, receives an immediate reward $r^{(t)}$  and thereafter transitions to a next state   $s^{\left(t+1\right)}$. This process is repeated, and in each time frame, the agent collects the tuple $e^{(t)}=\left(s^{(t)},a^{(t)},r^{(t)},s^{\left(t+1\right)}\right)$  as an experience.  The objective of  the RL-based system design is to find an optimal policy that maximizes the expected return, where the return $G^{(t)}$ is defined as the cumulative total discounted reward from an initial state $s^{(t)}$  and is mathematically expressed as 
\begin{equation}
G^{(t)} =\sum_{i=0}^{\infty} \gamma^{i} r^{(t+i)}\hspace{0.8cm}  0\le\gamma\le 1, 
\end{equation}
where the discount factor $\gamma$ describes the impact of future rewards relative to those in the immediate future. Under a certain policy $\pi\left(s^{(t)},a^{(t)}\right)$, the associated $Q$-function value is defined as  the total expected reward of executing action $a^{(t)}$ under state $s^{(t)}$, i.e.,
\begin{equation}
    Q^{\pi}\left(s^{(t)},a^{(t)}\right) =\mathbb{E}_{\pi}\left[G^{(t)} \mid s^{(t)}, a^{(t)}\right].
\end{equation}

The optimal  $Q$-function $Q^{\pi^{*}}\left(s^{(t)}, a^{(t)}\right)$  associated with the optimal policy $\pi^{*}\left(s^{(t)}, a^{(t)}\right)$ can be obtained by solving the Bellman equation recursively \cite{RL_book}. However, obtaining the optimal $Q$-function is challenging, particularly with the large
dimensional state and action space. To overcome this curse of dimensionality, a DQN denoted as $Q^{\pi}\left(s^{(t)},a^{(t)};\theta^{(t)}\right)$ parameterized by  $\theta^{(t)}$,  is trained to approximate $Q^{\pi}\left(s^{(t)}, a^{(t)}\right)$  from which the corresponding optimal policy $\pi^{*}\left(s^{(t)},a^{(t)}\right)$ can be extracted. Additionally, to stabilize the learning,  two DQNs are utilized during the training process: an evaluation network with weight parameters ${\theta^{(t)}}$ and a target network with parameters $\bar{\theta^{(t)}}$. The target network parameters are periodically updated by copying the weights from the evaluation network. The DRL agent continuously collects new experiences and stores them in the experience replay memory  $\mathcal{M}$, which works in a \textit{first-in-first-out} (FIFO) manner.  Then, by sampling a mini-batch $\mathcal{M}^{mini}$ of  experiences with length $\|B\|\,\left( B=\left\{e_1, e_2, \cdots, e_B\right\}\right)$  from $\mathcal{M}$, the DRL agent can update $\theta^{(t)}$ by adopting a proper optimizer to minimize the  following mean-squared loss function
\begin{equation}\label{loss_DQN}
\begin{aligned}
    {L}\,(\theta^{(t)}) = \frac{1}{\|B\|}\sum_{e\,\in\,{B}}\left\{\left(y^{(t)}(s^{(t+1)},r^{(t+1)})-Q^{\pi}(s^{(t)}, a^{(t)} ; {\theta^{(t)}})\right)^{2}\right\},
\end{aligned}
\end{equation}
where  \small $y^{(t)}\left(s^{(t+1)},r^{(t)}\right)=\Big\{r^{(t)}+\gamma\, \underset{a^{(t+1)}}{\max} Q^{\pi}\left(s^{(t+1)},a^{(t+1)}; {\bar{\theta^{(t)}}}\right)\Big\}$ \normalsize is the target Q-value generated by the target network.  Updating the DQN weights by using batches of randomly selected states from the experience replay memory rather than just the latest state breaks the temporal correlation among states, ultimately improving performance.

Deep-Q learning is well suited for discrete state and action spaces but cannot learn continuous stochastic policies. When the network has continuous state and action spaces, as is the case for our power allocation problem, the DDPG  algorithm is a suitable alternative, which can effectively deal with continuous optimization space \cite{24}. The DDPG algorithm works in an actor-critic manner: The evaluation actor network represents the parameterized policy network $\mu\left(s^{(t)}; \psi^{(t)}\right)$ with weights $\psi^{(t)}$, and the action is extracted via the deterministic policy as $a^{(t)}=\left[\mu\left(s^{(t)}; \psi^{(t)}\right)+\mathcal{W}\right]_{0}^{P_{\max }}$, 
where $\mathcal{W}$ represents the stochastic noise following a normal distribution. On the other hand, the evaluation critic network $Q^c\left(s^{(t)},a^{(t)};\phi^{(t)} \right)$ with weights $\phi^{(t)}$ represents the value-function network, which takes the current state $s^{(t)}$ and the deterministic action $a^{(t)}$ as inputs and outputs a single $Q$-value to estimate the long-term reward.  Similar  to the  DQN, target actor network $\mu\left( s^{(t)}; \bar{\psi^{(t)}}\right)$ with weights $\bar{\psi^{(t)}}$, and target critic network $Q^c\left(s^{(t)},a^{(t)};\bar{\phi^{(t)}} \right)$, with weights $\bar{\phi^{(t)}}$ are used to stabilize the learning process. The weight parameters $\phi^{(t)}$ of the critic network can be adjusted in a fashion similar to the DQN parameter
update by  minimizing the following critic loss function:
\begin{equation}\label{loss_DDPG}
 L(\phi^{(t)})=\frac{1}{\|B\|}\sum_{e\,\in\,{B}}\left\{\left(y^{(t)}\left(s^{(t+1)},r^{(t)}\right)-Q^c(s^{(t)}, a^{(t)} ; {\phi^{(t)}})\right)^{2}\right\},
\end{equation}
where \small$y^{(t)}\left(s^{(t+1)}, r^{(t)}\right)=\Big\{r^{(t)}+\gamma\,\underset{a^{(t+1)}}{\max} Q^c\left(s^{(t+1)},\mu(s^{(t+1)}; \bar{\psi^{(t)}}); \bar{\phi^{(t)}}\right)\Big\}$ \normalsize is the target critic network.

The actor network parameters $\psi^{(t)}$ can be updated by the gradient of the  critic network output w.r.t. the action $a^{(t)}$, multiplied by the gradient of the actor network  output w.r.t. its parameters  $\psi^{(t)}$, and averaged over a mini-batch  as follows:
\begin{equation} \label{actor_update}
\nabla_{\psi^{(t)}} J\left(\psi^{(t)}\right)=\frac{1}{\|B\|}\sum_{e\,\in\,{B}}\left\{\nabla_{\psi^{(t)}} \mu\left(s^{(t)}; \psi^{(t)}\right) \nabla_{a^{(t)}} Q^c\left(s^{(t)}, a^{(t)} ; \phi^{(t)}\right)\right\},
\end{equation}
where $\nabla_{\psi^{(t)}}$ denotes the gradient w.r.t. $\psi^{(t)}$ and $J\left(\psi^{(t)}\right)$ is the policy objective function to be maximized. Then, based on the stochastic gradient descent technique, the weights of the  evaluation DQN, the evaluation critic network, and the evaluation actor network can be updated as 
\begin{equation}
\begin{aligned}\label{critic_update}
&\theta^{(t)} \leftarrow \theta^{(t)}-\eta \nabla_{\theta^{(t)}} L\left(\theta^{(t)}\right), \\
&\phi^{(t)} \leftarrow \phi^{(t)}-\delta \nabla_{\phi^{(t)}} L\left(\phi^{(t)}\right), \\
&\psi^{(t)} \leftarrow \psi^{(t)}-\zeta \nabla_{\psi^{(t)}} J\left(\psi^{(t)}\right),
\end{aligned}
\end{equation}
where $\eta$, $\delta$,\hspace*{0.04cm} $\zeta$ are the learning rates of the evaluation DQN, evaluation critic network, and the evaluation actor network, respectively. The parameters of the target actor network and the target critic network are softly updated using $\bar{\psi^{(t)}} \leftarrow \tau \psi^{(t)}+(1-\tau) \bar{\psi^{(t)}},$ \hspace{0.2cm}$\bar{\phi^{(t)}} \leftarrow \tau \phi^{(t)}+(1-\tau) \bar{\phi^{(t)}}$ to slowly track the learned evaluation network parameters, where  $\tau \ll$ 1 \cite{RL_book}, \cite{mnih2015humanlevel}.

\subsection{DRL-Based Resource Allocation Framework }\label{distribted DRL}
In this section, we cast the $\cal{MMDEP}$ problem  (\ref{opt_problem}) as a sequential decision-making process by mapping the key elements from the RL framework to the $\cal{MMDEP}$ problem. Our proposed two-layered DQN and DDPG-based framework is elaborated first, followed by the derived event-triggered DRL-based algorithm in the following section. 

The  $\cal{MMDEP}$ problem inherently constitutes a hybrid discrete-continuous action space. The two-layered network structure at the RSU optimizes the discrete block length allocation policy and the continuous transmit power allocation policy simultaneously, as shown in Fig. \ref{fig:DRL}. For the first layer responsible for the discrete block length assignment, we propose a multi-DQN structure.  The key idea is to train a separate DQN for each block length allocation decision so that the total number of DQN structures equals the number of VUEs. In this way, the scale of the block length action space is reduced from  $(M_{D})^{K}$ to $K$x${M_{D}}$.
On the other hand, due to the continuous optimization space of the transmit power, we utilize the DDPG-based algorithm to assign the downlink continuous power in the second layer. This approach effectively addresses the challenge of quantization errors and the consequent performance degradation resulting from the discretization of continuous actions \cite{MetaV2X}. This two-layered structure is beneficial as it eases the implementation and enables the concurrent optimization of the two policies during the training phase for faster convergence.

\begin{figure*}[t]
 \centering
\includegraphics[width= 1\linewidth]{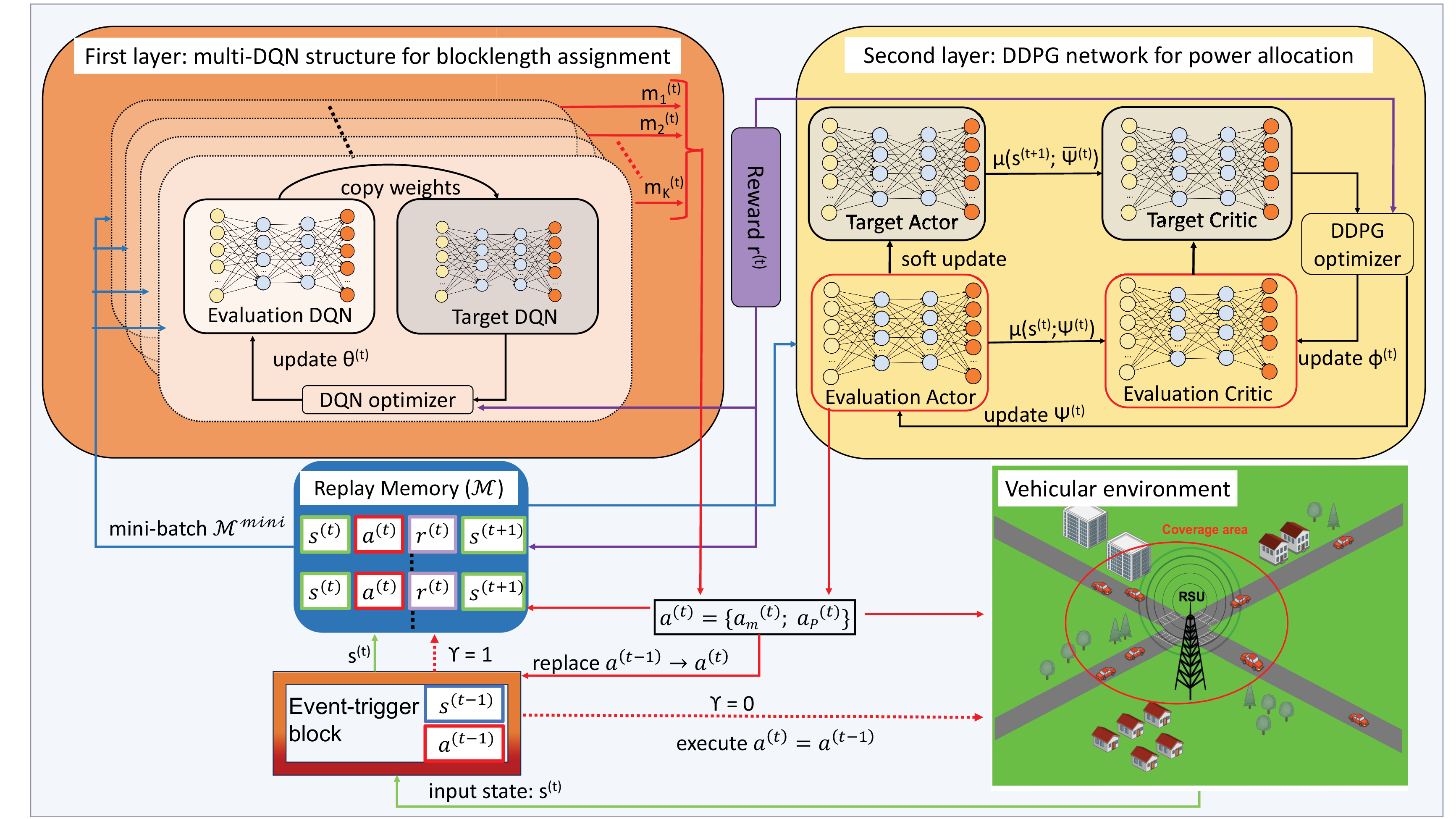}
\caption{Proposed event-triggered DRL framework for power control and block length allocation.} \label{fig:DRL}
\end{figure*}

Next, we introduce the RL framework in detail in terms of the state space, action space, and the reward function.
\begin{itemize}
\item  $\mathit{State \hspace{0.1cm}Space:}$  
Considering the channel correlation between adjacent frames, we utilize information from historical frames in addition to the instantaneous channel information, which offers essential data that can be used to optimize the resource allocation policies \cite{rr21}. 
\begin{enumerate}
\item  The  historical data comprises the downlink channel gains, $\mathbf{h}^{(t-1)}=  \left[h_{ 1}^{(t-1)}, h_{ 2}^{(t-1)}, \ldots, h_{ K}^{(t-1)}\right]$, previously allocated transmit powers, $\mathbf{P}^{(t-1)}=\left[P_{ 1}^{(t-1)}, P_{ 2}^{(t-1)}, \ldots, P_{ K}^{(t-1)}\right]$, and  previously assigned block lengths, $\mathbf{m}^{(t-1)}=\left[m_{ 1}^{(t-1)}, m_{ 2}^{(t-1)},\ldots, m_{ K}^{(t-1)}\right]$.
\item The instantaneous information includes the  channel gains, $\mathbf{h}^{(t)}=  \left[h_{ 1}^{(t)}, h_{ 2}^{(t)}, \ldots, h_{ K}^{(t)}\right]$, in the current time frame.
\end{enumerate}
\begin{equation}\label{state}
\begin{aligned}
&&{s^{(t)}}=\Bigg\{\underbrace{\mathbf{h}^{(t-1)},  \mathbf{P}^{(t-1)}\hspace{0.1cm}, \mathbf{m}^{(t-1)}}_{\text {historical information }}, \underbrace{\mathbf{h}^{(t)} }_{\text {instantaneous information }} \Bigg\}.
\end{aligned}
\end{equation}    
The cardinality of the state space is ($4 \times K$), depending on the number of VUEs.

\item $\mathit{Action\hspace{0.1cm}Space:}$ 
The DRL agent performs two actions in each time frame $(t)$: adjusting the block length, denoted as $a^{(t)}_{m}$, and setting the transmit power, denoted as  $a^{(t)}_{P}$. The combined action of the DRL agent  can be expressed as
\begin{equation}
 a^{(t)}=\left\{ a^{(t)}_{m}, a^{(t)}_{P}  \right\}= \Bigg\{m_k^{(t)}, P_k^{(t)}| \hspace{0.2cm} k = 1, \ldots, K\Bigg\}, 
\end{equation}
where the discrete value $m_k^{(t)} \in\{0,1, \ldots, M_D\}$ and the continuous value $P_k^{(t)} \in\left[0, P_{\text {max }}\right]$ represent the block length and transmit power allocated  to the $k^{th}$ VUE at time frame $t$, respectively.
\item  $\mathit{Reward\hspace{0.1cm}function:}$ The flexibility of RL reward design makes it particularly appealing for solving problems with difficult-to-optimize objectives.
To achieve the global objective (\ref{obj}), while considering the stringent latency and reliability requirements, we  design a reward function, which is expressed as

\begin{equation}
\label{reward}
\begin{aligned}
r^{(t)} = \Bigg\{ 
&-\alpha_1 \max_{k \in K}\hspace*{0.1cm} \left(\varepsilon_{1}^{(t)}, \ldots, \varepsilon_{k}^{(t)}, \ldots, \varepsilon_{K}^{(t)}\right) \\
&-\alpha_2 \max \left(\sum_{k \in \mathcal{K}}m_{k}^{(t)}P_{k}^{(t)}-P_{max}M_{D}, 0\right) \\
&-\alpha_3 \sum_{k \in \mathcal{K}}\left(\max \left(\epsilon_{k}^{(t)}-\epsilon_{th}, 0\right)\right) \\
&- \alpha_4 \max \left(\sum_{k \in \mathcal{K}} m_{k}^{(t)}-M_{D}, 0\right) \hspace{0.1cm} \Bigg\},
\end{aligned}
\end{equation}


where $\alpha_{i}, i \in\{1,2,3,4\}$, are the positive weights to balance the utility (obtained from the objective function) and the cost in terms of constraint violation. To align the reward function with the desired objective, we include the first term in Eq.(\ref{reward})  as the negative reward to be consistent with the objective in our system model. Additionally, the unsatisfied constraints of transmit power, reliability, and latency are incorporated into the reward as penalties. The weight coefficients are carefully selected as they impact the learning efficiency and convergence of the algorithm. In practice, the coefficients $\alpha_{i}$ are hyperparameters and need to be tuned
empirically. In our training, we set $\alpha_{1} = \alpha_{3} =1$,  $\alpha_{2} = \frac{1}{(P_{max}M_{D})}$ and $\alpha_{4} = \frac{1}{M_{D}}$. This hyperparameter selection approach balances the contribution of transmit power, reliability, and latency constraints to the desired objective, and enables stable updates of the  Q-values for improved convergence.  
\end{itemize}
\subsection{Event-triggered learning}\label{ETL}
In general, RL is studied by means of Markov decision processes (MDPs), where learning is performed in a time-triggered mechanism. The periodic training for learning may be unnecessary when the environmental fluctuation is negligible. Besides, training requires high-performance computing to reduce time costs, while the required hardware is expensive and power-hungry. Inspired by event-triggered control systems \cite{ET, ETO}, which activate computation or communications only when the system state deviates from the expected accuracy, we incorporate event-triggered learning into our proposed DRL framework to reduce computational cost and improve learning efficiency.

According to (\ref{jake}), for low Doppler frequency, the input states of the training agent, which includes the channel gains in consecutive time frames, are very similar due to the temporal correlation between the current and past channel conditions. This implies that the resource allocation decisions are likely to be the same in consecutive time frames.  As shown in Fig. \ref{fig:DRL}, based on the temporal correlation between the states of the agent, we introduce an event-trigger block before the initiation of the DRL process, which checks whether to initiate the learning process for new action selection or use the former action for resource allocation. The event-triggering block holds the previous state-action pair $\left\{s^{(t-1)},a^{(t-1)}\right\}$, and compares the former state of the agent against the current observed state. If the absolute difference is below the event-triggering threshold, then the former action is executed without initiating the learning process, thus reducing the use of computational resources over time frame $t$. Specifically, the triggering criterion is defined as follows:

\begin{equation} \label{ETL_EQ}
 \Upsilon  = \begin{cases}1, & \text { if }\left\|s^{(t)}-s^{(t-1)}\right\| \geq \upsilon\\ 0, & \text { otherwise }\end{cases}   
\end{equation}
where $\upsilon  \geq 0$ is the trigger threshold, $\Upsilon  =1$ means initiating the DRL process, and $\Upsilon  =0 $ means using former state-action pair. 
\vspace{-0.1cm}
\subsection{Proposed DRL-based Resource Allocation Algorithm}\label{proposed DRL}
In this section, we present our proposed event-triggered DRL algorithm for solving the  $\cal{MMDEP}$  problem. The proposed event-trigger learning and DRL-based algorithm for the joint block length assignment and transmit power allocation is denoted by Event-triggered DDPG-DQN, and is summarized in Algorithm \ref{DDPG_algo}, which includes three phases, i.e., initialization, random experience collection,  and event-triggered training and execution.

\begin{algorithm} 
\SetAlgoLined
\caption{ Proposed Event-triggered DDPG-DQN Algorithm } \label{DDPG_algo}
\nonl \textbf{Initialization:}\ \\
\textnormal{ Initialize the replay memory block, $\cal{M}$, and the event-trigger block with $s^{(t-1)}=0,\hspace{0.1cm} a^{(t-1)}=0$}\;
\textnormal{ Randomly initialize  $\theta^{(t)}$, $\psi^{(t)}$, $\phi^{(t)}$}\;
\textnormal{ Initialize the target network weight parameters $\bar{\theta^{(t)}} \leftarrow \theta^{(t)} , \bar{\psi^{(t)}} \leftarrow \psi^{(t)}, \bar{\phi^{(t)}} \leftarrow \phi^{(t)}$}\;
\nonl \textbf{Random experience collection:}\ \\
\For{$t=0,1, \ldots T_{rand}$ }{
    Execute the random joint action $a^{(t)}$\;
    Obtain the reward using (\ref{reward})\;
    Observe the new state $s^{(t+1)}$\;
    Store experience $e^{(t)}=\left(s^{(t)},a^{(t)},r^{(t)},s^{\left(t+1\right)}\right)$ in $\mathcal{M}$\;
\If{$t > B$ }{
\textnormal{Sample $B$ experiences  from  $\mathcal{M}$}\;
\textnormal{Update   the weights of evaluation networks  $\theta^{(t)}$, $\psi^{(t)}$ and $\phi^{(t)}$ with  (\ref{critic_update})}\;
}}
\nonl \textbf{Event-triggered training:}\ \\
\For{$t=T_{rand}, T_{rand}+ 1, \ldots \mathcal{T}_{train}$ }{
\textnormal{ Compute event-triggering condition $\Upsilon$ according to (\ref{ETL_EQ})}\;
 \If{$\Upsilon = 1$ }{
    \textnormal{Choose the block length action $a_{m}^{(t)}$ following the $\epsilon-$greedy policy}\;
    \textnormal{Generate the deterministic power action $ a_{P}^{(t)}=\left[\mu\left(s^{(t)} ; \psi^{(t)}\right)+\mathcal{W}\right]_{0}^{P_{\max }}$}\;
    \textnormal{$a^{(t)}=\left\{ a^{(t)}_{m}, a^{(t)}_{P}  \right\}$ \;}}
 \Else{Execute the previous  action $a^{(t-1)}$\;}
\textnormal{ Evaluate the reward using (\ref{reward}),  next state $s^{(t+1)}$ and store experience $e^{(t)}=\left(s^{(t)},a^{(t)},r^{(t)},s^{\left(t+1\right)}\right)$  in $\cal{M}$}\;
\textnormal{ Randomly sample $B$ experiences  from  $\mathcal{M}$\;}
\textnormal{ Update   the weights of evaluation networks  $\theta^{(t)}$, $\psi^{(t)}$ and $\phi^{(t)}$ with  (\ref{critic_update})}\;
}
\end{algorithm}

In the \textit{initialization} phase, $K$ evaluation DQNs $Q\left(s^{(t)},a^{(t)};\theta^{(t)}\right)$, and $K$ corresponding target DQNs $Q\left(s^{(t)},a^{(t)};\bar{\theta^{(t)}}\right)$ are constructed as a first layer at the RSU, meanwhile evaluation actor network $\mu\left( s^{(t)}; \psi^{(t)}\right)$, target actor network $\mu\left( s^{(t)}; \bar{\psi^{(t)}}\right)$,  evaluation critic network $Q^c\left(s^{(t)},a^{(t)};\phi^{(t)} \right)$, and the corresponding target critic network $Q^{c}\left(s^{(t)},a^{(t)};\bar{\phi^{(t)}} \right)$ are established as a second layer at the RSU. The weights $\theta^{(t)}, \psi^{(t)}, \phi^{(t)} $ of the evaluation networks are randomly initialized, whereas the target network weights $\bar{\theta^{(t)}}, \bar{\psi^{(t)}}, \bar{\phi^{(t)}}$  are initialized with the weights of the corresponding evaluation network weights, i.e., $\theta^{(t)}=\bar{\theta^{(t)}}$, $\psi^{(t)}=\bar{\psi^{(t)}}$, and $\phi^{(t)}=\bar{\phi^{(t)}}$, respectively (Lines $1-3$).

In the \textit{random experience collection},  we let the agent do the random walk for $T_{rand}$ time frames, which induces changes in the channel conditions and consequently allows the policy to observe
more variant states during its training. This procedure intuitively increases the algorithm's robustness to the changes in channel conditions. At the beginning of each time frame, the RSU receives the channel
gains from all VUEs to establish the current state $s^{(t)}$. The DRL agent selects a random transmit power and block length action, obtains the reward using (\ref{reward}), and transitions to the next state $s^{(t+1)}$ due to the channel variation in the next time frame. Meanwhile, the experience   $e^{(t)}=\left(s^{(t)},a^{(t)},r^{(t)},s^{\left(t+1\right)}\right)$ is stored in the experience  replay memory $\cal{M}$ (Lines $5-8$).  Note that the replay buffer memory capacity $\mathcal{M}$  is much larger than the batch size $B$, and RSU begins to sample from $\mathcal{M}$ when at least $B$ sample experiences are available. A mini-batch of experiences of size $B$ from  $\mathcal{M}$ is used to train the evaluation DQNs, the actor network, and the critic network, i.e., to minimize the DQN loss in (\ref{loss_DQN}) and the critic loss function in (\ref{loss_DDPG}). Afterwards, the agent updates the actor policy by the sampled policy gradient with (\ref{actor_update}), the evaluation network parameters with (\ref{critic_update}), and soft updates the target network parameters with those of the evaluation networks (Lines $10-11$).

In the \textit{event-triggered training} phase $\left( t \geq T_{rand}\right)$, at the beginning of time frame $t$, the RSU receives the channel gains from all VUEs. The current state of the agent  $s^{(t)}$ is passed to the trigger block to compute (\ref{ETL_EQ}) (Line $15$). The DRL process is initiated if  $\Upsilon =1$. The current state is sent to the multi-DQN structure (first layer) and the DDPG-based network (second layer) for block length and transmit power estimation, respectively. Specifically, the DRL agent outputs the block length allocation action $a^{(t)}_{m}$ using $K$ separate DQNs while following the $\epsilon-$greedy policy, and the deterministic power action $ a_{P}^{(t)}=\left[\mu\left(s^{(t)}; \psi^{(t)}\right)+\mathcal{W}\right]_{0}^{P_{\max }}$ using the actor network (Lines $16-18$).  Otherwise, the agent executes the previous action stored in the event-trigger block (Lines $19-22$). The agent obtains the reward using (\ref{reward}), the next state $s^{(t+1)}$, and stores the current experience tuple  $e^{(t)}$ in the replay buffer memory $\cal{M}$ (Line $24$). Then, by sampling a mini-batch of $B$ experiences, the agent trains the evaluation DQNs, the actor network, and the critic network in a manner similar to that described in the random experience collection phase  (Lines $25-26$).
Finally, upon convergence, the algorithm outputs the trained parameters; $\widetilde{\theta}^{(t)}$, $\widetilde{\psi}^{(t)}$, $\widetilde{\phi}^{(t)}$ for the evaluation DQNs,  the evaluation actor, and the evaluation critic network, respectively.

The major computational complexity of  Algorithm \ref{DDPG_algo} comes from training  $K$  evaluation DQNs,  and four DNNs in the actor-critic network structure. Assuming that the evaluation actor network, the critic network, and each DQN uses  $J$, $N$, and $L$ fully connected layers, respectively, the computational complexity can be calculated using
$O\Big(\sum_{j=0}^{J-1} u_j u_{j+1}+ \sum_{n=0}^{N-1} u_n u_{n+1} + \sum_{l=0}^{L-1} u_l u_{l+1} \Big),$
where $u_i$ denotes the number of units in $i$-th fully connected layer, and $u_0$ indicates the input size.


\section{Performance Evaluation} \label{sec:simulation}
In this section, we evaluate the performance of our proposed event-triggered DDPG-DQN based approach in comparison to the three state-of-the-art benchmark algorithms, namely, our proposed joint optimization algorithm ($\cal{JOA}$), the single-DQN scheme, the DDPG-DQN  scheme, and the random allocation algorithm.
The single-DQN scheme considers a single DQN for the joint transmit power and block length allocation by quantizing the power into ten discrete levels. This scheme is only available for discrete action spaces due to the inherent limitation of the DQN and is based on the adaptation of the joint DRL-based methodology proposed in \cite{joint_learning}. The  DDPG-DQN  scheme optimizes the continuous transmit power and block length in each time frame by utilizing the proposed DRL-based framework without the event-triggered learning mechanism. Finally, the random allocation algorithm chooses a random transmit power and block length action at the beginning of every time frame. 

To compare the performance of different learning schemes, the results are averaged over ten independent simulation runs with random initialization. The simulation setup constitutes two main phases, i.e., the training phase and the testing phase. The training phase lasts $ \mathcal{T}_{train}=20,000$ time frames, whereas the subsequent $5,000$ time frames are dedicated to the testing phase, where the well-trained DNNs are utilized to evaluate the performance of the different learning algorithms. To obtain the figures,  a second-order Savitzky-Golay smoothing filter with a window size of $503$ is utilized for smoothness. 
All simulations are performed on a $10$-Core Intel(R) Xenon(R) Silver $4114$, $2.2 \mathrm{GHz}$ system equipped with an Nvidia Quadro P2000 graphics processing unit (GPU). 

\subsection{Simulation Setup} \label{sec:Sim_setup}

\begin{table}[!t]
 \caption{\textsc{Simulation parameters}\label{table_sim}}
     \centering
     \scalebox{0.75}{
     \begin{tabular}{|c|c|}
\hline   \textbf{Parameter}  &  \textbf{Value}  \\
\hline  Maximum  transmit power  $P_{max }$ & 23 {dBm} \\
\hline  Noise power  $\sigma^{2}$ & -114 dBm \\
\hline   Vehicle receiver noise figure  & 9 dB \\
\hline  Vehicle speed $v_s$  & 60\hspace{0.05cm}km/hr \\
\hline  Vehicle placement model  &   spatial Poisson   \\
\hline   Lane width  & 4 {m} \\
\hline  { Carrier frequency $f_{c}$} & 2 {GHz} \\
\hline  { Shadowing standard deviation $\sigma_{sh}$} & 3 {dB} \\ 
\hline  { Radius } & 500 {m} \\
\hline  { RSU antenna height } & 25 {m} \\
\hline  { RSU antenna gain } & 8 {dBi} \\
\hline  { RSU receiver noise figure } & 5 {dB} \\
\hline  { Distance between RSU and highway } & 35 {m} \\
\hline  { Vehicle antenna height } & 1.5 {m} \\
\hline { Vehicle antenna gain } & 3 {dBi} \\
\hline { Payload size $L$ } &  100 {bits} \\
\hline SNR threshold $\gamma_{th}$ & 0  {dB}\\
\hline
\end{tabular}}
\end{table}

\begin{table}[t]
 \caption{\textsc{Hyperparameters for actor network} \label{table_actor}}
\centering
\small
\scalebox{0.74}{
     \begin{tabular}{|c|c|c|c|c|c|}
\hline Layers & Input & $\mathit{L_1^{(a)}}$ & $\mathit{L_2^{(a)}}$ & $\mathit{L_3^{(a)}}$ & Output \\
\hline Neuron number & ($4K$) & 500 & 250 & 120 & 1 \\
\hline Activation function & Linear & Relu & tanh & Relu & Sigmoid \\
\hline Action noise $\mathcal{W}$  & \multicolumn{5}{c|}{Standard Gaussian distribution $\mathcal{W}$(0,0.1) }\\
\hline Mini-batch size $B$ & \multicolumn{5}{c|}{128} \\
\hline Random experience collection period $T_{\text {rand }}$ & \multicolumn{5}{c|}{$3 \times B$} \\
\hline Replay buffer memory capacity & \multicolumn{5}{c|}{500} \\
\hline
\end{tabular}}
\end{table}

We investigate the performance of our proposed algorithm against several schemes by simulating a vehicular network with multiple VUEs. In particular,  we follow the simulation setup for a freeway model described in 3GPP TR 36.885, which describes in
detail the vehicular channel models, vehicle mobility, and vehicular data traffic \cite{r16}. In particular,  $N$  V2I links are initiated by $N$ vehicles and the K V2V
links are formed between each vehicle with its surrounding
neighbors. The vehicles are distributed according to  Poisson distribution on a multi-lane highway with the BS at the center and at a fixed distance from the highway. The considered highway is a $3\times2$ lane with a fixed lane width.  The vehicle density on the highway is determined by the vehicle speed $v_s$, and the average inter-vehicular distance is $2.5\hspace{0.001cm} \times \hspace{0.01cm}v_s$,  which is considered fixed in our simulation results. The large-scale fading is modeled using the WINNER path-loss model $PL(d)=128.1 + 37.6 \log _{10}(d)$, 
where $d$ is the distance  in $[\mathrm{km}]$ of the VUE from the RSU \cite{r19}. Fast fading has been modeled using Rayleigh fading with the time correlation between successive time frames, as represented by Jake’s model. The important simulation parameters are provided in Table \ref{table_sim}.

\begin{table*}[t]
 \caption{\textsc{HYPERPARAMETERS FOR CRITIC NETWORK AND DEEP-Q NETWORK}  }
 \label{table_critic}
\small
\centering
\scalebox{0.75}{
     \begin{tabular}{|c|c|c|c|c|c|c|c|c|c|c|c|}
\hline  DNN type& \multicolumn{5}{c|}{Critic network} & \multicolumn{5}{c|}{DQN} \\
\hline Layers & Input & $\mathit{L_1^{(c)}}$ & $\mathit{L_2^{(c)}}$ & $\mathit{L_3^{(c)}}$ & Output & Input& $\mathit{L_1^{(DQN)}}$ & $\mathit{L_2^{(DQN)}}$ & $\mathit{L_3^{(DQN)}}$ & Output \\
\hline Neuron number & ($4K$)+1 & 500 & 250 & 120 & 1  & ($4K$) & 500 & 250 & 120 & $M_{D}$ \\
\hline Activation function & Linear & Relu & tanh & Relu & Linear & Linear & Relu & tanh & Relu & Linear \\
\hline Optimizer  & \multicolumn{10}{c|}{RMSPropOptimizer} \\
\hline Mini-batch size $B$ & \multicolumn{10}{c|}{128} \\
\hline Replay buffer memory capacity  & \multicolumn{10}{c|}{500} \\
\hline Discount factor $\gamma$  & \multicolumn{10}{c|}{0.5}  \\
\hline Initial learning rate $\delta^{(0)}$   & \multicolumn{10}{c|}{0.05} \\
\hline Learning decay rate $\alpha_L$   & \multicolumn{10}{c|}{0.0001}  \\
\hline Initial exploration rate $\epsilon^{(0)}$   & \multicolumn{10}{c|}{0.1} \\
\hline $\epsilon$-decay rate $\alpha_\epsilon$  & \multicolumn{10}{c|}{0.0001}  \\
\hline
\end{tabular}}
\end{table*}

The hyperparameter values for the actor DNN, the critic DNN, and the DQN are tabulated in Table \ref{table_actor} and  \ref{table_critic}, respectively. The tabulated parameters are fine-tuned to reach the desired results. The DNNs used to implement the actor network, the critic network, and the DQN have a similar architecture with one input layer, three hidden layers, and one output layer. The rectifier linear unit (ReLu) and hyperbolic tangent (tanh) functions are used alternatively as the activation function since this scheme leads to faster convergence, as per our observation. The input states are normalized and pre-processed to optimize the performance. In particular,  logarithmic transformation is applied to the input states for event-triggered control and data processing before feeding them to the DNNs. Logarithmic representation is preferred since channel amplitudes often vary by order of magnitude. The output of the actor network is scaled to be between 0 and $P_{max}$. To encourage exploration during training, Gaussian action noise $\mathcal{W}$ with mean 0 and standard deviation of 0.05  is added to the actor network output. The replay buffer capacity and the mini-batch size are set to 500 and 128, respectively.   The $\epsilon$-greedy algorithm is initialized with $\epsilon^{(0)}=$ 0.1 and the level of exploration is adaptively reduced according to $\epsilon^{(t+1)}=\max \left\{0,\left(1-\alpha_\epsilon\right) \epsilon^{(t)}\right\}$, where $\alpha_\epsilon=$0.0001 is the $\epsilon$-decay rate. 

\subsection{Analysis of Event-triggered learning} \label{ANALYSIS}
\begin{figure}
     \centering
     \begin{subfigure}[t]{0.4\textwidth}
         \centering
         \includegraphics[width=\textwidth]{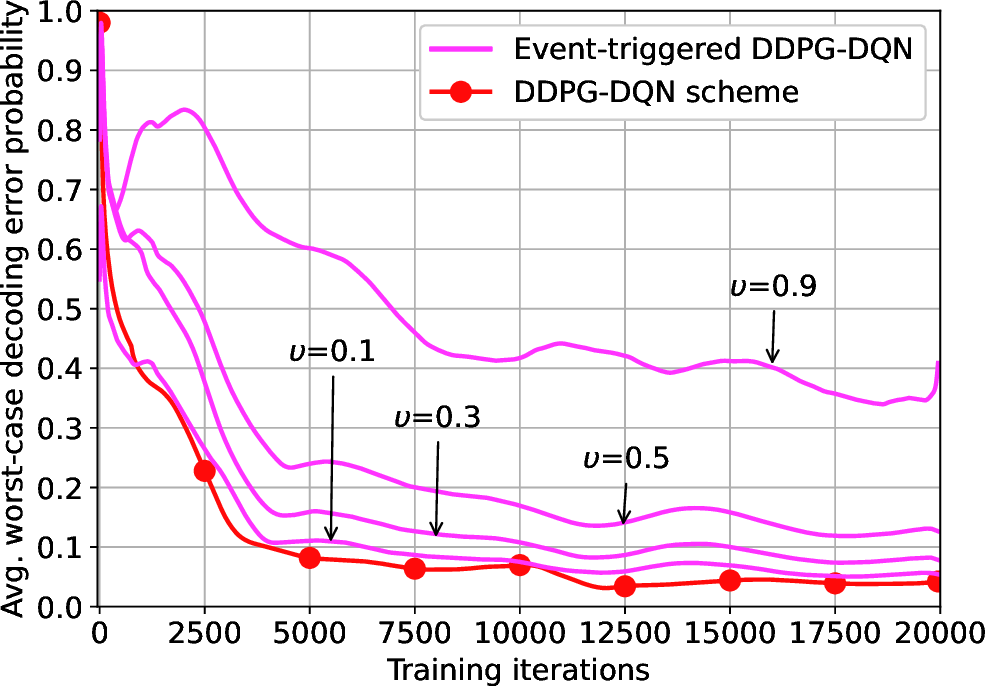}
         \caption{}
         \label{N10a}
     \end{subfigure}
     \begin{subfigure}[t]{0.4\textwidth}
         \centering
         \includegraphics[width=\textwidth]{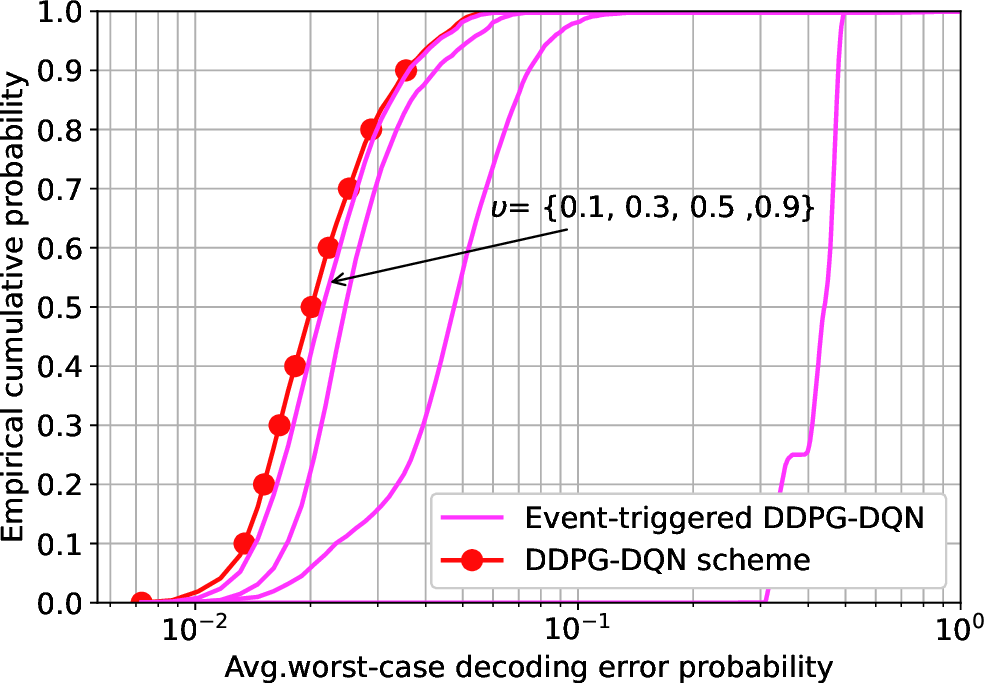}
         \caption{}
         \label{N10b}
      \end{subfigure}
      \begin{subfigure}[t]{0.4\textwidth}
        \centering
        \includegraphics[width=\textwidth]{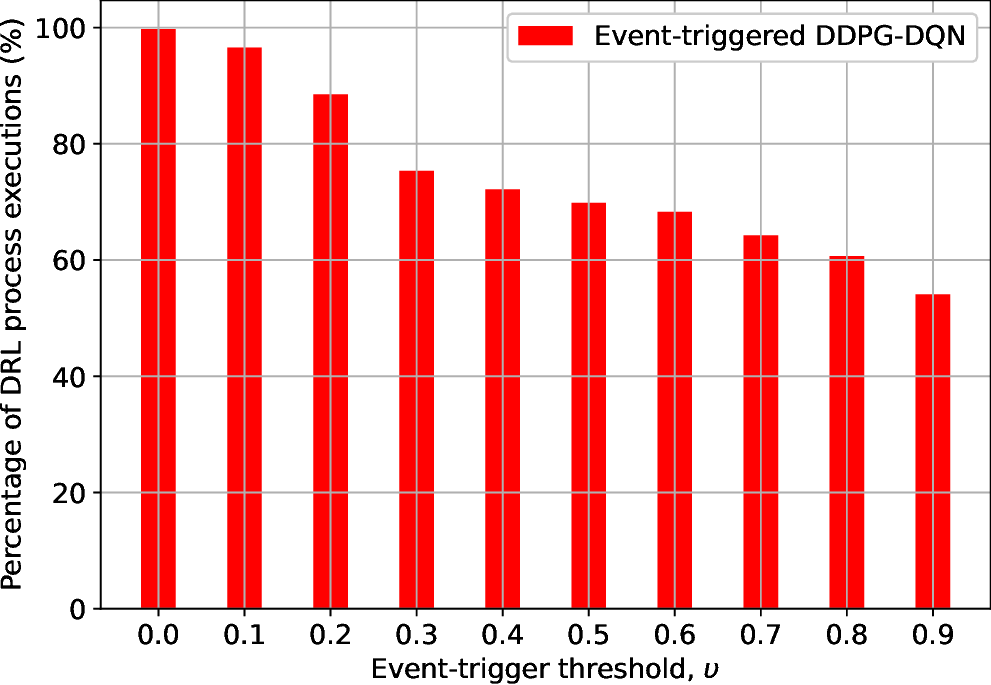}
        \caption{}
        \label{bar}
      \end{subfigure}
     \caption{Average worst-case decoding-error probability performance of the algorithms, where (a) shows the training performance, (b) shows the testing performance for different event-trigger threshold values $\upsilon =\left\{0.1, 0.3, 0.5, 0.9 \right\}$, and (c) shows the percentage of DRL executions for varying trigger threshold values, with $K=20$, $P_{max}=23$ dBm, and $M_{D}=300$ symbols.}
     \label{N10}
     \vspace{-0.87cm}
\end{figure}

We investigate the impact of changing the trigger threshold value $\upsilon$  on the training convergence and number of DRL process executions of the proposed event-triggered DDPG-DQN scheme. Note that the input states are normalized and scaled before evaluating the trigger condition.

Fig. \ref{N10} demonstrates the impact of different event-trigger threshold values on the training convergence and complexity of the proposed algorithm in terms of DRL process executions. As shown in Fig. \ref{N10a} and \ref{N10b}, for low values of threshold, i.e., $\upsilon = 0.1$, the reliability performance of the proposed event-triggered DDPG-DQN scheme is very close to that of the  DDPG-DQN scheme and the average performance gap between the two schemes is less than  $\sim 1\%$. On the other hand, for very high threshold values, i.e., $\upsilon = 0.9$, the reliability performance significantly deteriorates in comparison to the  DDPG-DQN scheme, and the event-triggered DDPG-DQN scheme achieves only  $\sim 11.12\%$ of the reliability performance provided by the  DDPG-DQN scheme. This performance degradation can be explained by the agent's reduced exposure to useful information critical for optimizing its action. Fig. \ref{bar} illustrates the percentage of DRL executions for varying trigger threshold values in the training phase. Increasing the threshold value results in the reduction of the number of DRL process executions. For instance, with $\upsilon = 0.9$, the DRL agent only trains the network for  $\sim 50\%$ of the total training duration but provides poor reliability performance. Combining the results from  Fig. \ref{N10a}, \ref{N10b}, and \ref{bar}, we can observe an interesting trade-off between the reliability performance and computation complexity in terms of the percentage of DRL executions.   A moderate value of the threshold, e.g., $\upsilon=0.3$, can achieve performance close to the DDPG-DQN scheme with an average performance gap less than $\sim 3\%$  while reducing the percentage of DRL executions by up to 24$\%$.  Therefore, for the subsequent simulation results, we fix the value of the event-trigger threshold to $\upsilon =$ 0.3. Unless otherwise stated, the default parameters are set as follows: $P_{max}=$  23 {dBm}, $M_{D}=$ 300, and $\upsilon =$ 0.3.

    

\subsection{Performance Comparison of Algorithms}
Fig. \ref{N5a} illustrates the training convergence behavior of the different algorithms. At the start of data transmissions, the average worst-case decoding error probability is identical for the proposed event-triggered DDPG-DQN algorithm and the random resource allocation algorithm since the proposed algorithm randomly assigns the power and block length to each VUE to collect experiences. From the figure, both the proposed event-triggered DDPG-DQN  and the  DDPG-DQN schemes converge to around $95\%$ of the average worst-case decoding-error probability obtained using the joint optimization method. The average worst-case decoding error probability using the proposed algorithm decreases rapidly compared to the single-DQN scheme and converges after around 13,000 time frames. In contrast, the single-DQN scheme exhibits slow and poor convergence.  This is mainly due to the high complexity of the action space of the single-DQN scheme, where a single DQN is used to determine the power and block length allocation actions for all VUEs. Fig. \ref{N5b} demonstrates the effectiveness of the proposed algorithm by empirically validating the performance of the different algorithms in the testing stage. The proposed algorithm outperforms the single-DQN scheme and achieves on average 21$\%$ higher reliability compared to the single-DQN scheme on the testing data. This means that by exploiting the well-trained DNNs, the proposed algorithm can effectively learn the patterns of the environment and provide close to optimal results on the environmental states that did not appear before.

\begin{figure}[t]\label{N5}
     \centering
    \begin{subfigure}[t]{0.4\textwidth}
         \centering
    \includegraphics[width=\textwidth]{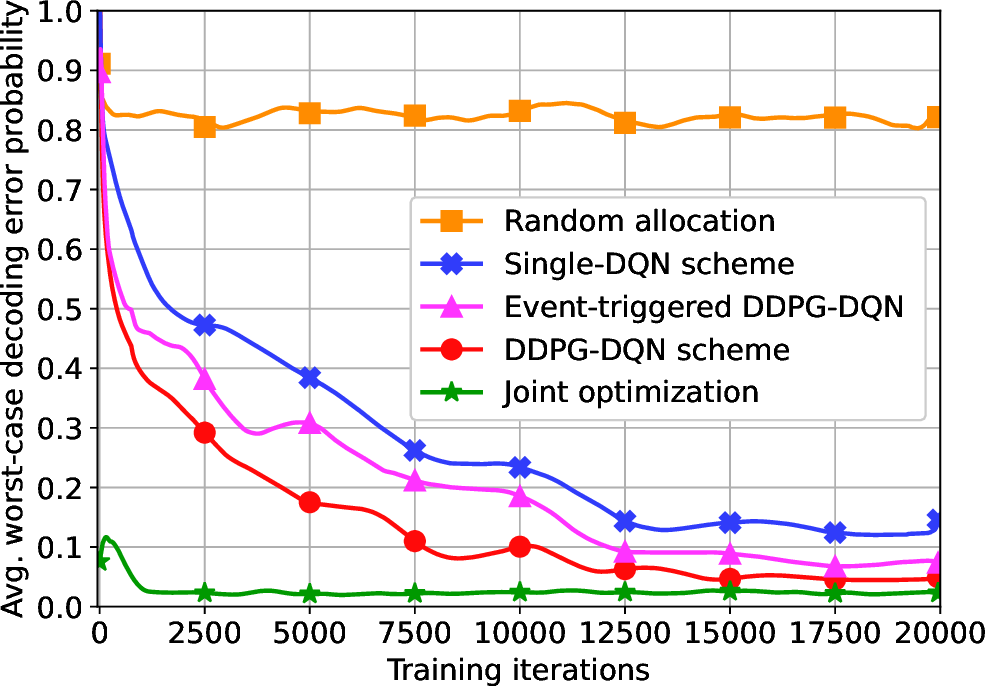}
         \caption{Training}
         \label{N5a}
\end{subfigure}
\begin{subfigure}[t]{0.4\textwidth}
         \centering
    \includegraphics[width=\textwidth]{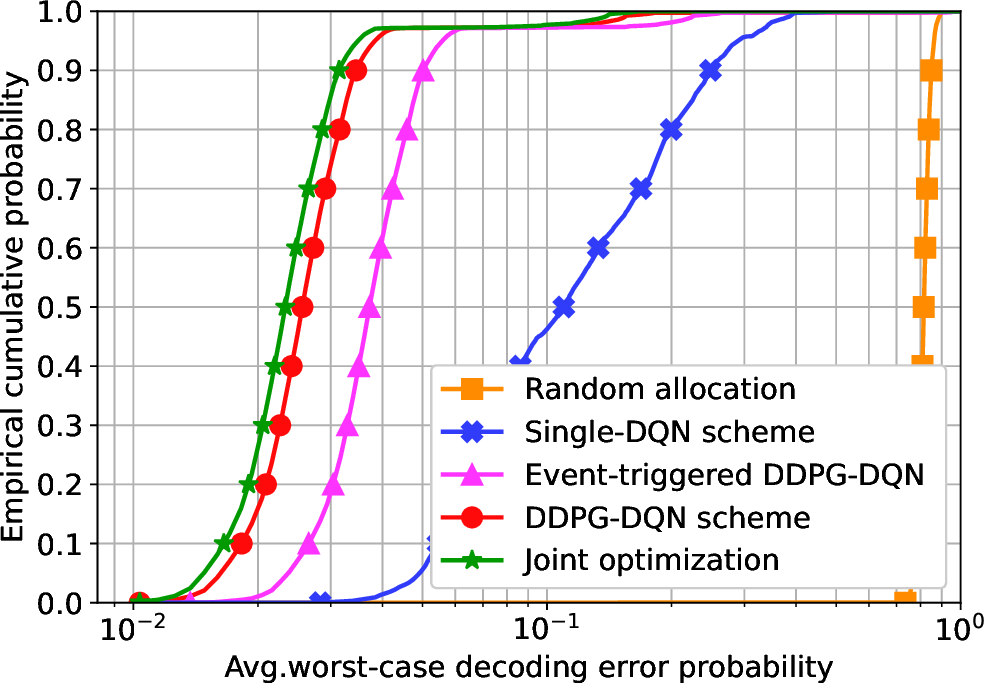}
         \caption{Testing- Empirical CDF}
         \label{N5b}
\end{subfigure}
\caption{Average worst-case decoding-error probability performance of the algorithms.} 
\end{figure}

Fig. \ref{symbols_trend}  illustrates the impact of maximum available symbols $M_{D}$ on the average worst-case decoding-error probability. The error probability decreases with the increase in total available block length as each VUE gets a larger percentage of the total available block length resource, leading to better reliability performance. The event-triggered DDPG-DQN scheme remarkably achieves $\sim 95 \%$ of the reliability performance of the joint optimization scheme and $\sim 98 \%$  of the reliability performance offered by the  DDPG-DQN algorithm, while outperforming the single-DQN for  $K=$ 5 and $K=$ 10 VUEs. The decoding-error probability increases with the number of VUE pairs $K$ since more VUEs are competing for the fixed available block length resource.

Fig. \ref{power_trend}   shows the effect of maximum transmit power on the average worst-case decoding-error probability. Higher transmit power leads to better
reliability performance, as expected. The decoding-error probability increases with the number of  VUEs due to the limited power resource and degradation of the available SNR with high road traffic density. The proposed event-triggered DDPG-DQN algorithm performs close to the joint optimization scheme, and the performance gap between the event-triggered DDPG-DQN and the  DDPG-DQN scheme is less than 2$\%$ for different values of  $P_{max}$. On the other hand,  the single-DQN scheme cannot keep up with the proposed approach in terms of reliability performance. This is mainly due to the large action space of single-DQN compared to the proposed approach, which utilizes a multi-DQN structure.  Besides, the quantization of power into discrete levels using the single-DQN approach discards certain pertinent information that may be crucial in determining the most suitable option for power allocation.

Fig. \ref{vehicle_trend}   provides the average reliability performance of the algorithms for different numbers of vehicles. Clearly, an increase in the number of VUEs results in a degradation of the reliability performance of all the algorithms since more vehicles are competing for the same limited resources. The random allocation strategy cannot satisfy the reliability requirement at all due to the random allocation of power and block length.   The single-DQN scheme exhibits a degraded performance for all considered vehicular densities because as the vehicular network size scales up, the network topology as well as the number of possible actions increase rapidly, making it very challenging to explore the whole action space to find out the optimal action choice. On the other hand, the proposed event-triggered DDPG-DQN achieves on average 95$\%$ of the reliability performance of the joint optimization solution, and the performance gap between the event-triggered DRL scheme and the DDPG-DQN scheme is below 5$\%$  for different vehicular densities, demonstrating the robustness of the proposed algorithm to network size.

\begin{figure}[t]
\includegraphics[width=0.4\textwidth]{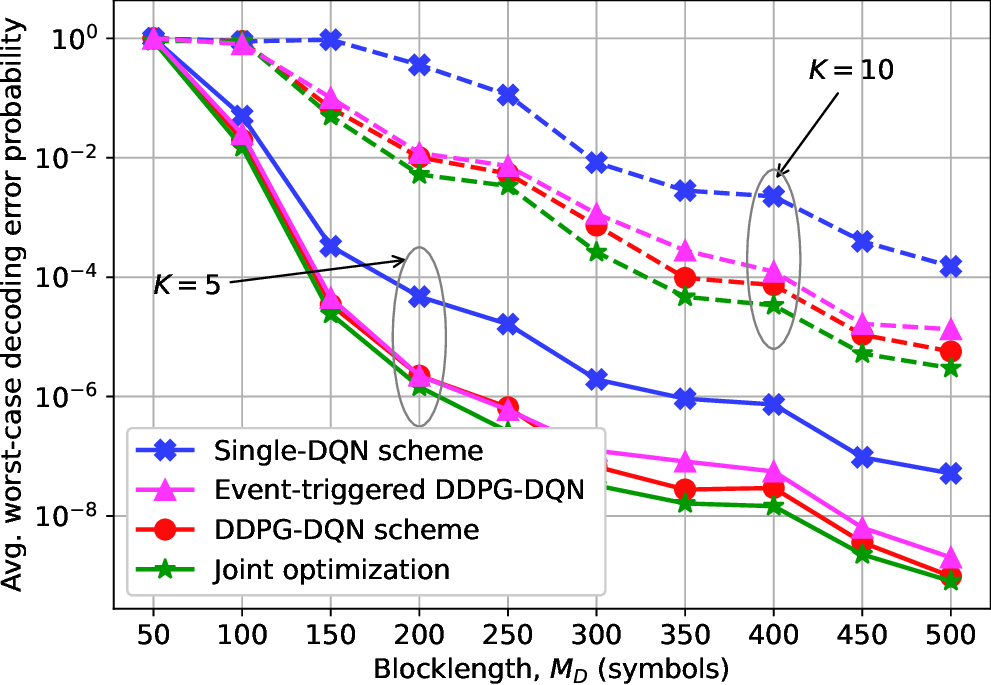}
\caption{ Optimized average  worst-case decoding-error probability of the algorithms
versus the number of symbols $M_{D}$.}
\label{symbols_trend}
\end{figure}

\begin{figure}[t]
\includegraphics[width=0.4\textwidth]{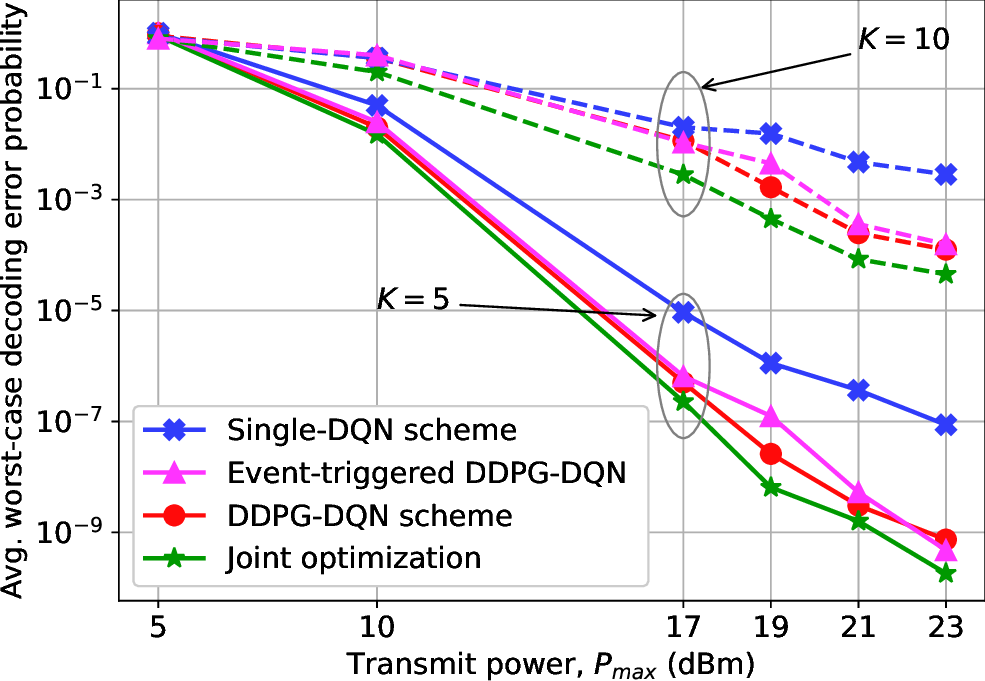}
\caption{ Optimized average worst-case decoding-error probability of the algorithms
versus the maximum transmit power $P_{max}$.}
\label{power_trend}
\end{figure}

Fig. \ref{run-time} illustrates the average run-time of the algorithms for different network sizes. As the network size increases, the run-time complexity increases
significantly for the joint optimization algorithm due to its associated second-order polynomial growth rate in $K$, which limits its scalability. In contrast, once trained well, the DRL agent on average can determine its power and block length assignment action using the proposed algorithm in less than \SI{0.8}{\milli\second}. Another observation is that the single-DQN scheme has the lowest run-time since a single trained network executes the two actions. However, the reduced run-time for the single-DQN scheme comes at the cost of poor convergence and degraded reliability performance as the network size increases. The proposed event-triggered DDPG-DQN scheme and the DDPG-DQN scheme utilize the same network architecture and have almost the same run-time for allocating resources for different vehicular densities. It should be pointed out that the proposed event-triggered DDPG-DQN scheme is trained using $\sim 75\%$ of total training executions. Table (\ref{table_time})  shows the average time complexity of various algorithms computed across training steps and different network sizes (K=\{15, 20, 25, 30, 35, 40\} vehicles). From the table, the average time to train the proposed two layers (multiple DQNs and the actor-critic network) is around \SI{11.2} {\milli\second}, which is significantly lower than that of DDPG-DQN, and the single-DQN schemes. Additionally, the average time required for calculating transmit power and block length allocation with the proposed solution is around \SI{0.8} {\milli\second}, significantly less than the average calculation time of approximately \SI{93} {\milli\second} needed for employing the joint optimization algorithm. Hence, the proposed framework can allocate the radio resources within the millisecond-level transmission delay requirement of real-time URLLC V2X applications. Since the computational capability of RSUs in real-world V2X networks exceeds that of the computers utilized in simulations, the average time required for training the evaluation networks and determining transmit power and block length solutions can be further reduced in practical networks.

\begin{figure}[!t]
\includegraphics[width=0.4\textwidth]{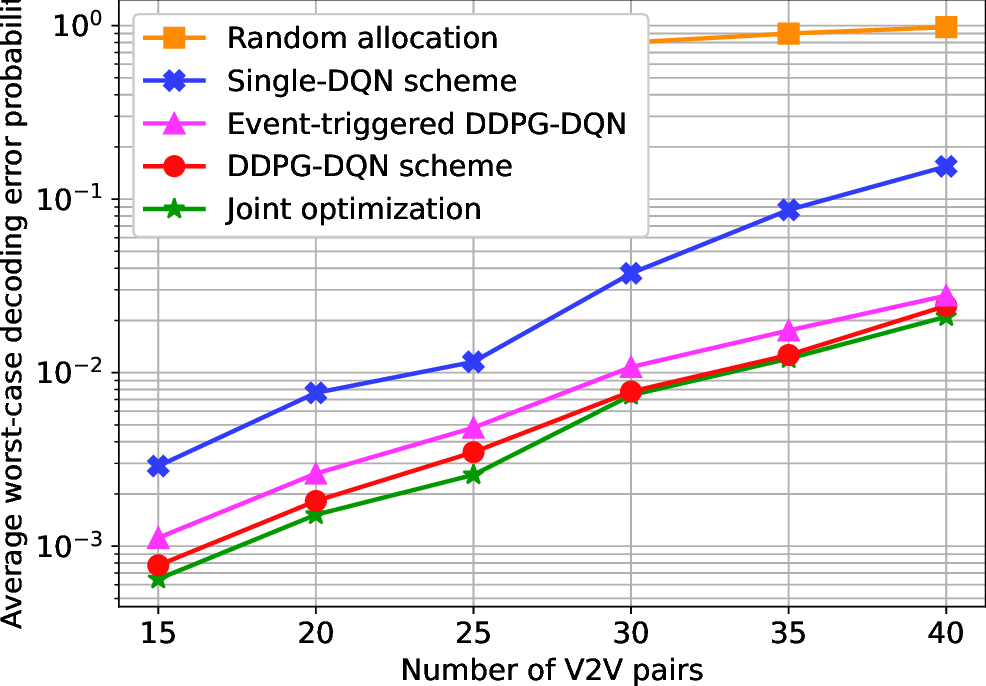}
\caption{ Optimized average worst-case decoding-error probability of the algorithms
versus the number of VUEs $K$.}
\label{vehicle_trend}
\end{figure}

\begin{figure}[!t]
\includegraphics[width=0.4\textwidth]{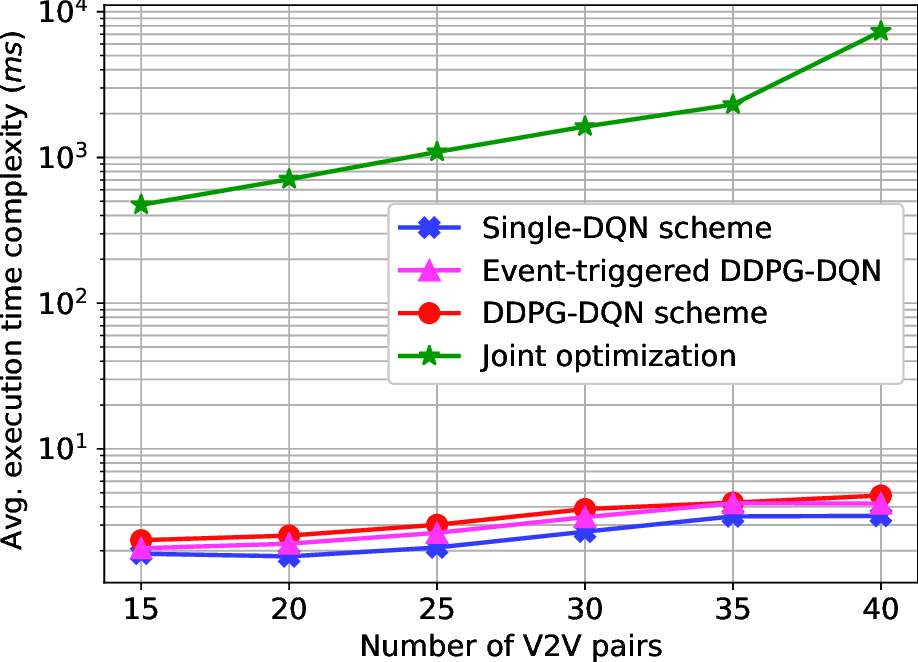}
\caption{Average execution time complexity of the algorithms for different number of VUEs $K$.}
\label{run-time}
\end{figure}
\vspace*{-5mm}

\subsection{Impact of Vehicular Mobility and Reliability Performance:}
 So far, we have presented results under fixed vehicle speed, which implies a static Doppler frequency and, therefore, a constant correlation between channels in adjacent time frames. However, in practical scenarios,  VUEs may change their speed as they move around, and it is essential to consider vehicle mobility's impact on DRL algorithms' performance.  The convergence rate of the proposed DRL-based algorithm depends on the Doppler frequency. As we decrease Doppler frequency from $f_{D}= 112$ Hz ($v_s = 60 \, \text{km/hr}$) to $f_{D}= 28$ Hz ($v_s = 15\, \text{km/hr}$), the average decoding error-probability performance remains unchanged, but the convergence time drops from 13,000 time frames to  6,000 time frames. Intuitively, this decline in convergence time with lower Doppler frequencies is related to the variability of the states observed by the agent, which is proportional to the Doppler frequency.

 To evaluate the reliability performance of the trained model in the deployment (testing phase), we use the metric of \textit{average network-wide reliability} ($\mathcal{ANR}$) defined as the  average ratio of the number of vehicles with  decoding error probability   no greater than $\epsilon_{\text{max},k}$, to the total number of vehicles, formulated as: 
\begin{equation} \label{eq:ANR}
\mathcal{ANR} = \frac{1}{\mathcal{T}_{test}}\sum_{t=1}^{\mathcal{T}_{test}} \sum_{k=1}^K \frac{\mathbb{I}\left(\epsilon_{\text{max},k} > \epsilon_k^{(t)}\right)}{K}
\end{equation}
where $\mathbb{I}$ represents the indicator function, and the averaging is performed across the testing steps $\mathcal{T}_{test}$. In our evaluation, we consider a network with $K = 5$, $P_{\text{max}} = 23 \, \text{dBm}$, $M_D= 300$ symbols. We define two sets of error probability constraints  $\varepsilon_{{max}}^{*} = [10^{-2}, 5 \times 10^{-2}, 10^{-3}, 5 \times 10^{-3}, 5 \times 10^{-2}]$, and $\varepsilon_{{max}}^{**} = [10^{-6}, 5 \times 10^{-5}, 10^{-4}, 5 \times 10^{-5}, 5 \times 10^{-4}]$, representing the  QoS requirements for different VUEs. For $\varepsilon_{{max}}^{*}$, the proposed scheme achieves an $\mathcal{ANR}$ of 96\%  compared to  87\% and 35\% $\mathcal{ANR}$ of single-DQN  and random allocation schemes, respectively. When considering more strict reliability requirements ($\varepsilon_{{max}}^{**}$), the $\mathcal{ANR}$ decreases for all schemes due to smaller values of the target errors impacting the network-wide reliability negatively. The proposed scheme still achieves an $\mathcal{ANR}$ of 88\%, significantly higher than the  69\% and 23\%  $\mathcal{ANR}$ of the single-DQN and random allocation schemes, respectively.  

\begin{table}[!t]
    \centering
    \caption{\textsc{Average time complexity}\label{table_time}}
    \renewcommand{\arraystretch}{1.2} 
    \scalebox{0.7}{
    \begin{tabular}{|c|c|c|}
    \hline
    \textbf{Scheme} & \makecell{\textbf{Average Training}\\ \textbf{Time} (\si{\milli\second})}  & \makecell{\textbf{Calculation with } \\ \textbf{trained network }(\si{\milli\second})} \\
    \hline
    Event-triggered DDPG-DQN & 11.2 & 0.8 \\
    \cline{1-1}\cline{2-2}\cline{3-3} 
    DDPG-DQN & 21.7 & 0.8 \\
    \cline{1-1}\cline{2-2}\cline{3-3} 
    Single-DQN & 41.3 & 0.6\\
    \hline
    \end{tabular}}
    \label{tab:training_times}
\end{table}

\subsection{\textit{Discussion on Implementation and Future Extensions}:}

  \textit{Training procedure:} In the proposed event-triggered DDPG-DQN algorithm, offline training is employed to mitigate the computational overhead of learning the mapping between input states and output actions.  Subsequently, during online network deployment, this mapping is executed using the pre-trained DNNs, bypassing the need for intensive computations \cite{suggested}.  Offline training can be performed using a digital twin of the vehicular network, incorporating network topology, channel models, and QoS requirements. The event-triggered DDPG-DQN algorithm can utilize this twin for offline initialization and network training at the central server. This approach allows the agent to explore in a simulated environment, mitigating real-world risks \cite{DT}.
 
 \textit{Analysis of constraint violations:}
 Our analysis of constraint violations in the testing stage reveals that the DRL agent does not always satisfy the power, aggregate delay, and reliability constraints. However, the well-trained agent satisfies the constraints in $\sim 95 \%$ and  $\sim 90 \%$ of the total testing executions for  \(K=5\) and \(K=20\) VUEs, respectively, with  $P_{max}=$  23 {dBm}, $M_{D}=$ 300 symbols. Ensuring zero-constraint violations in real-time operations remains a significant challenge in reinforcement learning.  On the one hand, providing theoretical guarantees for feasibility during online execution is crucial for DRL-based URLLC  applications. On the other hand, finding the solution
of constrained DRL is more challenging than standard DRL-based solutions. Constraint-guided DRL requires the
agents to maximize the long-term reward while satisfying the constraints as the long-term cost. In this regard, primal-dual methods \cite{primal}  and Teacher-Student learning framework \cite{teacher} offer a promising approach for constrained DRL-based solutions. The primal-dual method can be applied to optimize the constraints in the dual domain to achieve the target balance among the objective (long-term reward) and the long-term cost \cite{suggested}.  In the Teacher-Student framework, the student (DRL agent) seeks guidance from the teacher (domain-specific optimization-based algorithm) to select an action before executing it in the wireless environment, thereby accelerating the training process and safeguarding the agent against actions that would violate specific QoS requirements. 

 \textit{Dynamic network topology:} In practical scenarios, the number of vehicles under the coverage of RSU changes dynamically. To accommodate this variability, state encoding adjusts for changes in the number of VUEs under RSU coverage \cite{state_encoding}.  If the number of VUEs exceeds the trained agent's state space dimensions, information from the extra VUEs is not included in the state.  Instead, the RSU establishes the corresponding DNNs for the incoming vehicle and transfers the weights of the already-trained model to the newly established DNNs. Moreover, if one or more vehicles leave the coverage area, the agent's state is padded with zeros. It is worth pointing out that nearby vehicles often experience similar channel quality and environment observations. Therefore, the system's performance does not degrade significantly when new vehicles arrive or leave, as they can be allocated resources designated for the neighboring vehicles. Further, to avoid modifying the structure of the DNNs and retraining all the parameters with a changing number of vehicular links, graph neural networks (GNNs) are a promising avenue to explore for obtaining scalable solutions \cite{GNN}.

 \textit{Explainable AI:} For centralized learning-based algorithms, as is the case with our proposed framework, explainable AI (XAI) can be leveraged to further reduce the model's complexity and improve its robustness by explaining the learned policies during the training process \cite{khan2023}. XAI can alleviate the DRL model complexity and speed up the convergence time by offering methods for eliminating less important input features, aggregating identical states based on state abstraction in Markov decision processes, and utilizing model compression techniques. XAI can further improve the interpretability and robustness of DRL model decisions by providing methods for testing the system’s robustness, identifying potential outliers, and verifying the credibility of the decision outcomes. Additionally, to enhance the adaptability of the proposed DRL-based solution to new environment scenarios and tasks requiring computation-intensive training, emerging schemes such as model-agnostic meta-learning can be incorporated into the proposed framework \cite{MetaV2X}.

 \section{Conclusion } \label{sec:conclusion}
In this paper, we have considered the problem of joint resource allocation with short packet transmission in a URLLC-enabled V2X network. To ensure fair reliability among the vehicles, we have developed an optimization framework to minimize the maximum decoding error probability among the vehicles subject to stringent latency and reliability requirements in terms of the limited block length and tolerable decoding error probability. We have proposed an optimization theory based solution strategy using the joint convexity analysis and evaluation of the KKT optimality conditions to obtain the joint solution.  We have designed a computationally efficient  DQN and DDPG-based framework that optimizes both the discrete block length and continuous power allocation variables simultaneously under practical URLLC constraints. Additionally, we have incorporated event-triggered learning into our proposed DRL-based methodology, resulting in an interesting trade-off between the reliability performance and computation complexity in terms of the percentage of DRL process executions.  Simulation results have demonstrated that by incorporating the event-trigger mechanism, the proposed algorithm can achieve   $\sim 98 \%$  and $\sim 95 \%$  of the reliability performance offered by the  DDPG-DQN algorithm and the joint optimization scheme,
respectively, while remarkably reducing the DRL executions by up to 24$\%$.

\bibliographystyle{ieeetr}
\bibliography{bib_URLLC}

\end{document}